\documentclass[12pt]{article}
\large 
\usepackage[margin=0.8in]{geometry} 

\usepackage[utf8]{inputenc} 
\usepackage[T1]{fontenc}    
\usepackage{hyperref}       
\usepackage{url}            
\usepackage{booktabs}       
\usepackage{amsfonts}       
\usepackage{nicefrac}       
\usepackage{microtype}      
\usepackage{xcolor}         
\usepackage{amsmath}
\usepackage{enumitem}
\usepackage{pifont}
\usepackage{float}
\usepackage{graphicx}
\usepackage{subfigure}
\usepackage{subcaption}
\usepackage{graphics}
\usepackage{epsfig}
\usepackage[toc,page]{appendix}
\usepackage{booktabs,multirow}
\usepackage{amsmath,amssymb,mathtools} \usepackage{algorithm} \usepackage{algpseudocode} 

\usepackage{threeparttable} 
\usepackage{makecell}    

\usepackage{amsthm}
\usepackage{amssymb}

\usepackage{algorithm}
\usepackage{algpseudocode}

\allowdisplaybreaks

\newcommand{\algorithmicinitialize}{\textbf{Initialize:}}
\newcommand{\INITIALIZE}{\item[\algorithmicinitialize]}

\newtheorem{theorem}{\textbf {Theorem}}

\newtheorem{lemma}{\textbf {Lemma}}

\newtheorem{assumption}{\textbf {Assumption}}

\DeclareMathOperator*{\argmin}{arg\,min}

\DeclareMathOperator{\Orth}{Orth}

\begin{document}
\title{MALT: Lightweight Curvature-Aware Muon via Diagonal Preconditioning}

\author{%
Tongle Wu\textsuperscript{1}\thanks{Equal contribution.}
\quad
Huanyu Dong\textsuperscript{1}\footnotemark[1]
\quad
Ying Sun\textsuperscript{1}
\quad
Ziye Ma\textsuperscript{2}
\\[0.7em]
\parbox{0.90\textwidth}{%
\centering
\small
\textsuperscript{1}School of Electrical Engineering and Computer Science,\\
The Pennsylvania State University, University Park, PA, USA
\\[0.35em]
\textsuperscript{2}Department of Computer Science,\\
City University of Hong Kong, Hong Kong SAR, China
}%
}

\maketitle

\begin{abstract}
    
     Muon has recently emerged as a promising alternative to AdamW for language model pretraining by orthogonalizing momentum matrices using Newton–Schulz iterations. Although Muon mitigates gradient anisotropy, it does not explicitly account for the curvature geometry of the loss landscape and may therefore remain sensitive to curvature anisotropy. We bridge this gap by proposing \textbf{MALT}(\textbf{M}uon \textbf{A}ugmented by \textbf{L}ightweight \textbf{T}wo-sided preconditioning), which uses lightweight diagonal preconditioners to reduce the sensitivity of Muon to curvature anisotropy. Specifically, MALT uses two-sided diagonal preconditioners with low memory and computational overhead to approximately capture the curvature geometry of the loss landscape. It orthogonalizes the preconditioned momentum using Newton–Schulz iterations and maps the result back to define the update direction, while norm grafting controls the update magnitude. To improve the robustness of MALT to stochastic gradient noise, we further propose \textbf{MALTER}(\textbf{MALT} with {A}daptive st\textbf{E}psize \textbf{R}escaling). Convergence guarantees are provided for \textbf{MALT} in the stochastic non-convex setting. Experiments on GPT-2 Small, Medium, and Large pretraining show that the proposed methods outperform Muon while maintaining nearly the same memory footprint and wall-clock time.

\end{abstract}

\section{Introduction}
Recent progress in large language models (LLMs) has substantially accelerated the development of modern artificial intelligence \cite{bommasani2021foundation,zhao2023survey,minaee2024large}. Models such as GPT \cite{openai2026gpt55}, DeepSeek \cite{deepseekai2025deepseekv32}, LLaMA \cite{metaai2025llama4}, and Gemini \cite{geminiteam2025gemini25} have demonstrated strong capabilities across a wide range of tasks. However, scaling laws \cite{kaplan2020scaling} suggest that further performance gains of LLMs often require scaling up the model size, the amount of training data, and the training compute budget, which leads to substantial costs in training time, GPU memory \cite{rajbhandari2020zero}, communication \cite{narayanan2021efficient}, and energy consumption \cite{patterson2021carbon}. As a result, developing efficient optimizers becomes a critical component to reduce these training costs for large models \cite{ranganath2026navigating,liu2025muon}.

For decades, Adam \cite{kingma2015adam} and its variant AdamW \cite{loshchilov2019decoupled} have remained the default optimizers for training deep neural networks and LLMs. Adam provides elementwise adaptivity by maintaining a second moment estimate. This is equivalent to applying a diagonal preconditioner to the vectorized matrix parameter, neglecting correlations across rows and columns \cite{gupta2018shampoo}. Optimizers such as Shampoo \cite{gupta2018shampoo} and SOAP\cite{vyas2025soap} are subsequently proposed to address this limitation by constructing structure-aware preconditioners that exploit the matrix structure of the parameters, rather than relying on elementwise rescaling. For matrix parameter updates, the gradient or momentum matrix may have a highly unbalanced singular spectrum, causing the updates to be dominated by a few singular directions and slowing convergence \cite{lau2025polargrad,huang2026spectra}.  The recently proposed matrix optimizer Muon \cite{jordan2024muon} addresses this challenge by using Newton–Schulz iterations to approximately orthogonalize the momentum matrix, thereby flattening its singular spectrum and producing a nearly isotropic update direction.

However, Muon does not explicitly address curvature anisotropy of the loss landscape. Curvature anisotropy describes how sharp or flat the training loss landscape is across different directions, and is reflected by the condition number of the Hessian matrix of the loss function \cite{nocedal2006numerical,ghorbani2019hessian,sagun2018empirical}.  When the curvature anisotropy is high, using an update with the same magnitude may be too aggressive along sharp directions but too conservative along flat directions. High curvature anisotropy is known to slow down the convergence of first order optimization methods, whose rate depends directly on the condition number of the Hessian matrix \cite{nocedal2006numerical,ye2024preconditioning}. 
Figure~\ref{fig:compare-muon-precondition} (a) illustrates the effect of curvature anisotropy on Muon using a toy example. We consider solving the matrix quadratic regression problem $ \min \limits_{\boldsymbol X} f(\boldsymbol X) := \frac{1}{2}\left\|\boldsymbol A \boldsymbol X \boldsymbol B-\boldsymbol C \right\|_F^2 $ under a warmup-stable-decay (WSD) learning rate scheduling, one that is widely adopted in training modern language models  \cite{ibrahim2024simple,wen2024understanding}. We observe that as the Hessian condition number of $f\left( \boldsymbol X\right)$ increases, the convergence of Muon becomes substantially slower.
\begin{figure}[htbp]
\centering
\subfigure[Muon]{
\begin{minipage}[t]{0.5\linewidth}
\centering
\includegraphics[width=3.5in]{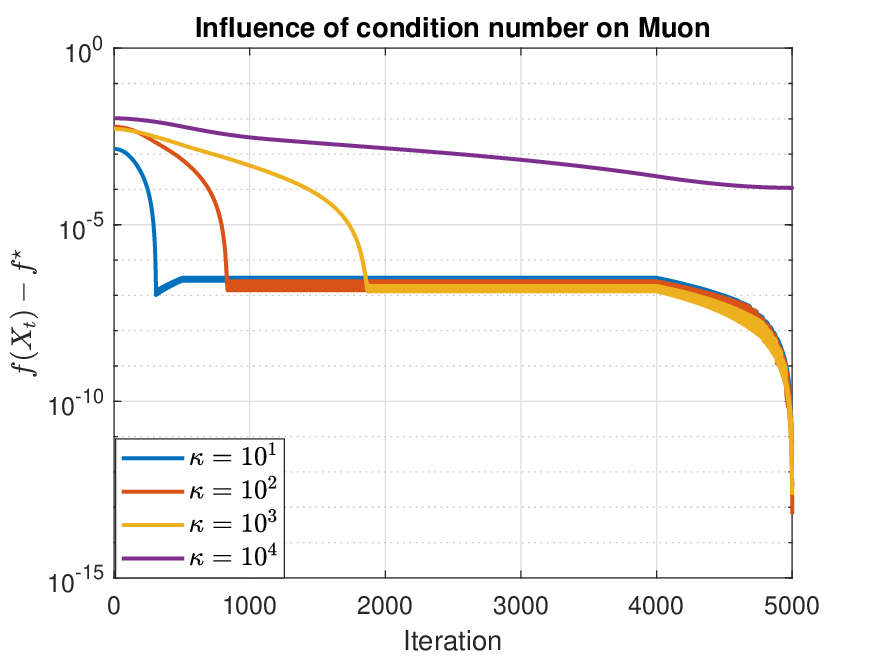}
\label{compare-1}
\end{minipage}%
}%
\subfigure[Preconditioned Muon]{
\begin{minipage}[t]{0.5\linewidth}
\centering
\includegraphics[width=3.5in]{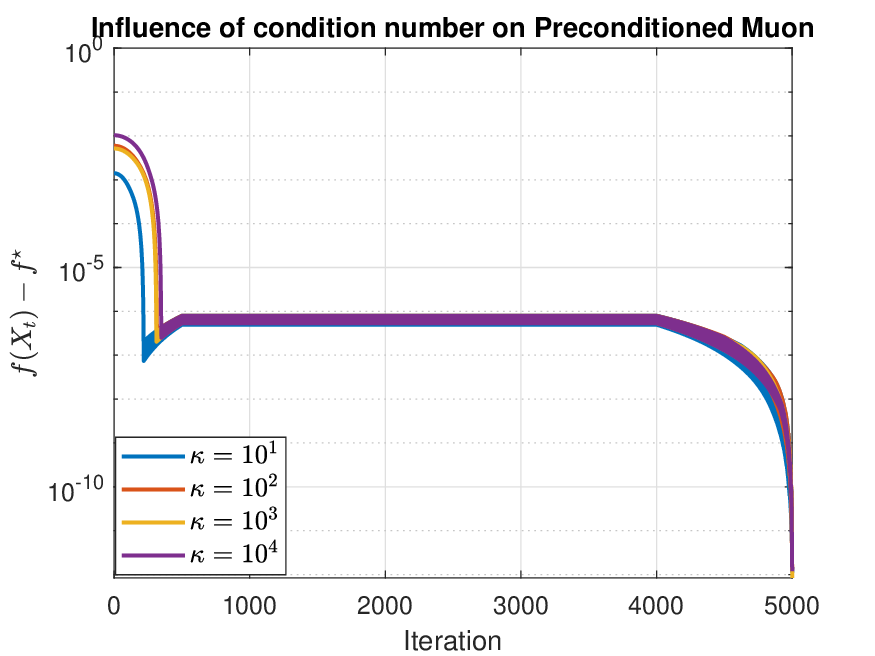}
\label{compare-2}
\end{minipage}%
}%
\centering
\caption{Comparison of Muon with and without curvature-aware preconditioning under the same WSD learning-rate schedule. Detailed setting for this simulation is provided in Appendix~\ref{app:simulation-setting}.}
    \label{fig:compare-muon-precondition}
\end{figure}

This sensitivity of Muon to curvature anisotropy motivates us to propose a new optimizer MALT to incorporate curvature-aware preconditioning into Muon. As demonstrated in Figure~\ref{fig:compare-muon-precondition} (b), curvature-aware preconditioned Muon is less affected by the condition number and converges substantially faster for is ill-conditioned loss. Our main contributions are summarized as follows:
\begin{itemize}

    \item We develop a formulation of Muon in a preconditioned space based on a local quadratic approximation of the loss. A positive definite curvature surrogate induces a transformed space in which the quadratic term becomes isotropic. Motivated by this formulation, MALT orthogonalizes the preconditioned momentum and maps the resulting direction back to the original parameter space. This formulation allows preconditioning and orthogonalization to address curvature anisotropy and gradient anisotropy simultaneously. 

    \item We develop a lightweight diagonal preconditioning scheme for Muon. It maintains exponential moving averages of row-wise and column-wise squared gradient norms to construct left and right diagonal preconditioners, with the goal of preserving matrix structure without introducing excessive extra memory and computational cost. To stabilize the update scale, MALT further uses norm grafting to rescale update magnitude. In addition to proposing MALT, we establish its convergence guarantees for stochastic nonconvex optimization. 
    
   \item To make MALT more robust to stochastic gradient noise, we  propose an adaptive variant of it named MALTER. MALTER introduces a single scalar stepsize for the grafted update direction, which is adaptively adjusted using norm based Adam-type noise adaptation in preconditioned space. Compared to MALT, MALTER shows improved empirical performance with negligible extra memory and computation overhead.

   \item We evaluate MALT and MALTER on pretraining GPT-2 Small, Medium, and Large models. Both methods consistently outperform Muon while costing nearly the same memory footprint and wall-clock training time. Among all compared optimizers, MALTER achieves the best overall performance. The performance gain of MALTER over Muon is comparable to that of Muon over AdamW.
   This shows that noise adaptive stepsize control and preconditioned orthogonalization are both important and complementary mechanisms for optimizer design.

\end{itemize}

\section{Related Work}

\noindent \textbf{Muon and its variants.}
Muon has recently attracted attention as a matrix optimizer that mitigates gradient anisotropy through Newton–Schulz orthogonalization. Existing variants of Muon mainly improve it from three aspects: how orthogonalization is
computed, how the orthogonalized direction is scaled, and how its
magnitude is distributed across rows or columns.

The Newton–Schulz iteration introduces several matrix multiplications for each Muon update, which can become costly in large scale training. The first line of work focuses on reducing the computational cost of orthogonalization. Existing approaches reduce this overhead through Newton–Schulz implementations optimized for modern hardware \cite{zhang2026gram}, more efficient approximation schemes for the polar factor \cite{grishina2025accelerating,amsel2025polar}, or low-rank parametrization \cite{refael2026sumo,he2025low}. 

The second line of work focuses on controlling the scale of Muon updates. Since orthogonalization flattens the singular spectrum, the choice of update magnitude is particularly important for Muon. Some recent studies interpret Muon as a steepest descent method under a spectral norm constraint and suggest that the stable step size depends on the nuclear norm of the momentum \cite{lau2025polargrad,nguyen2026spectral}. In addition, several variants combine the orthogonalization used in Muon with adaptive scaling based on RMSProp, AdaGrad, Adam, or trust regions to improve training efficiency and reduce sensitivity to stochastic gradient noise \cite{si2025adamuon,song2026decoupling,li2026variance,zhang2026namo,liu2026muon2,cheng2026trasmuon}.



Another line of work exploits structure across the rows and columns of Muon updates. Several studies observe that Muon updates can exhibit substantial norm imbalance across rows and columns  \cite{li2025normuon,zhang2026muonplus}, which may cause some neurons to receive persistently small updates and eventually contribute little to the network output 
\cite{aurora2026}. These methods normalize individual rows or columns after Newton-Schulz orthogonalization to improve the performance of Muon. Other methods either decompose each weight matrix into row magnitudes and directions \cite{lion2026muown}, or replace Newton-Schulz orthogonalization with normalization applied separately to each row \cite{deng2026rmnp,yuan2026nora}. 
MuonEq \cite{chang2026muoneq} is most closely related to our method. It equilibrates the momentum before orthogonalization, thereby reducing its condition number and improving the numerical accuracy of Newton-Schulz. However, MuonEq focuses on improving the conditioning of the momentum matrix and does not explicitly account for curvature anisotropy in the loss landscape.

\noindent \textbf{Preconditioned methods for mitigating curvature anisotropy.} The classical Newton method uses the Hessian matrix of the loss to precondition gradients and achieves faster convergence by alleviating the second order curvature anisotropy \cite{nesterov2018lectures}. However, for large scale problems, directly using the Hessian matrix is infeasible due to the prohibitive cost of computing, storing, and inverting this matrix \cite{lau2025polargrad}. This bottleneck has motivated the development of a broad range of practical alternatives that approximate the curvature information at reduced costs. For example, adaptive gradient methods such as AdaGrad \cite{duchi2011adaptive}, RMSProp \cite{tieleman2012rmsprop}, and Adam \cite{kingma2015adam} use elementwise gradient second moments to construct inexpensive diagonal preconditioners to rescale the gradients element-wisely. Memory-efficient variants of adaptive gradient methods such as Adafactor \cite{shazeer2018adafactor}, SM3 \cite{anil2019memory}, and Adam-mini \cite{zhang2025adammini} further reduce the cost of storing optimizer states by factorizing, compressing, or block-sharing these second moment of gradient. Other optimizers construct more direct approximations of curvature. For example, AdaHessian \cite{yao2021adahessian} and Sophia \cite{liu2024sophia} use estimates of the Hessian diagonal as preconditioners, whereas K-FAC \cite{martens2015optimizing} approximates the Fisher information matrix of each layer using Kronecker products. Matrix aware optimizers such as Shampoo \cite{gupta2018shampoo}, CASPR \cite{duvvuri2024caspr}, and SOAP \cite{vyas2025soap} exploit the matrix or
tensor structure of parameters by constructing structured preconditioners along rows and columns.
Nevertheless, anisotropy in matrix gradients
\cite{lau2025polargrad} remains underexplored in these methods.



\noindent \textbf{Most relevant work.} Concurrent works FISMO \cite{xu2026fismo} and Mousse \cite{zhang2026mousse} also proposed curvature-aware Muon variants. In particular, FISMO incorporates anisotropic geometry through a Fisher metric with a Kronecker factorization and formulates the orthogonalized update as a trust region problem under this Fisher induced metric. Mousse instead uses gradient statistics with a Kronecker structure, as in Shampoo, to construct a whitened space in which Muon style polar updates are performed. Newton–Muon \cite{du2026newton} is also closely related to our work, as it derives a curvature-aware Muon update from a Newton type surrogate and introduces a right preconditioner based on the second moment of the layer inputs.

Unlike these methods, MALT uses diagonal preconditioners on both sides with low computational and memory overhead and updates them at every optimization step. This preconditioner preserves the row and column structure of matrix parameters while avoiding the storage of dense curvature factors and costly matrix operations such as eigendecomposition and matrix inversion.


\section{Diagonally Preconditioned Muon}

\subsection{Motivation}
We consider minimizing the following unconstrained optimization problem:
\begin{align}\label{loss-fun}
\min_{\boldsymbol X \in \mathbb R^{m \times n}} f(\boldsymbol X),
\end{align}
where $f: \mathbb R^{m \times n} \to \mathbb R$ is a differentiable training loss function and $\boldsymbol X \in \mathbb R^{m \times n}$ is a matrix parameter.
Muon \cite{jordan2024muon} without momentum with exact orthogonalization updates $\boldsymbol X $ as
\begin{align}\label{muon-update}
\boldsymbol X_t = \boldsymbol X_{t-1} - \eta_t \operatorname{Orth} \left( \nabla f\left( \boldsymbol X_{t-1} \right) \right),
\end{align}
where $ \operatorname{Orth}\left( \boldsymbol X\right): = \boldsymbol U \boldsymbol V^\top $ denotes the orthogonalization of $\boldsymbol X \in \mathbb R^{m \times n}$ with $\boldsymbol X = \boldsymbol U\boldsymbol S \boldsymbol V^\top$ being its reduced singular value decomposition (SVD). This orthogonalization flattens the singular spectrum of the gradient $\nabla f(\boldsymbol X_{t-1})$ by replacing its nonzero singular values with ones, 
thereby reducing gradient anisotropy. To reduce the cost in training large models, it is often approximately computed using the Newton--Schulz iterations.

The Muon update \eqref{muon-update} does not explicitly incorporate curvature information of the loss landscape, which can make Muon sensitive to curvature anisotropy as demonstrated in Figure \ref{compare-1}. Nevertheless, the ill-conditioned Hessian that captures curvature anisotropy is widely observed in training neural networks \cite{sagun2016eigenvalues,sagun2017empirical,ben2024high} and language models \cite{shen2025convergence,zhang2024transformers,tang2025overlooked}.   


\begin{algorithm}[t]
\caption{MALT Update}
\label{algo-1}
\begin{algorithmic}[1]
\Require learning rate $\eta$, momentum $\mu_1,\mu_2 \in [0,1)$, batch size $B$ and damping constant $\epsilon$.
\INITIALIZE $\boldsymbol W_0 \in \mathbb R^{m\times n},  \boldsymbol M_0=\boldsymbol 0_{m\times n}$, $\boldsymbol l_0=\boldsymbol 0_m, \boldsymbol r_0 = \boldsymbol 0_n$.

\For{$t=1,\ldots,T$} 
\State Sample a minibatch $ \left| \mathcal B_t \right| = B$ and calculate stochastic gradient $\boldsymbol G_t = \frac 1 B \sum_{i \in \mathcal B_t }\nabla f\left( \boldsymbol W_{t-1};\xi_{i}\right)$
\State $\boldsymbol M_t \gets \mu_1 \boldsymbol M_{t-1}+(1-\mu_1) \boldsymbol G_t$
\State $\boldsymbol l_t \gets \mu_2 \boldsymbol l_{t-1}+(1-\mu_2)\left[ \left\| \boldsymbol G_t\left[1,:\right]  \right\|_F^2;\cdots; \left\| \boldsymbol G_t\left[m,:\right]  \right\|_F^2 \right]$
\State \State $\boldsymbol r_t \gets \mu_2 \boldsymbol r_{t-1}+(1-\mu_2)\left[ \left\| \boldsymbol G_t\left[:,1\right]  \right\|_F^2;\cdots; \left\| \boldsymbol G_t\left[:,n\right]  \right\|_F^2 \right]$
\State $\boldsymbol L_t \gets \operatorname{Diag}\big((\boldsymbol l_t+\epsilon)^{-1/8}\big)$, 
       $\boldsymbol R_t \gets \operatorname{Diag}\big((\boldsymbol r_t+\epsilon)^{-1/8}\big)$
\State $\widetilde {\boldsymbol M}_t \gets \boldsymbol L_t \boldsymbol M_t \boldsymbol R_t$
\State $\boldsymbol O_t \gets \Orth(\widetilde {\boldsymbol M}_t)$
\State $\boldsymbol D_t \gets \boldsymbol L_t \boldsymbol O_t \boldsymbol R_t$
\State $\boldsymbol W_{t}\gets \boldsymbol W_{t-1} -\eta \frac{\|\boldsymbol O_t\|_F}{\|\boldsymbol D_t\|_F+\epsilon}\boldsymbol D_t$
\EndFor

\State \Return $\boldsymbol W_T$
\end{algorithmic}
\end{algorithm}

\subsection {Muon in Preconditioned Space}
This motivates us to consider a local quadratic approximation of the loss function to encode curvature anisotropy. Unlike the use of the Hessian to capture second order information for strongly convex functions in the classical optimization literature, the Hessian of a nonconvex loss in deep learning or language models is not guaranteed to be positive semidefinite and may even be indefinite \cite{dauphin2014identifying,ghorbani2019hessian}. Consequently, directly minimizing a Hessian-based quadratic model may not yield a descent direction.

Moreover, explicitly forming, storing, or inverting the Hessian is prohibitively expensive in large-scale training \cite{lau2025polargrad}. Therefore, instead of using the exact Hessian, practical curvature-aware optimizers \cite{duchi2011adaptive,amari1998natural,martens2015optimizing,schraudolph2002fast,martens2010deep,gupta2018shampoo} often rely on tractable surrogates to approximately capture the curvature geometry of loss and have achieved effective empirical performance in training large scale models.


Therefore, it is reasonable to introduce a curvature-aware preconditioner
$
\mathcal B_t:\mathbb R^{m\times n}\rightarrow \mathbb R^{m\times n},
$
which serves as a tractable surrogate for the local curvature geometry around the current iterate \(\boldsymbol X_t\). 
With this curvature-aware surrogate, the local quadratic approximation of the training loss
\(f:\mathbb R^{m\times n}\rightarrow \mathbb R\) at
\(\boldsymbol X_t\) can be formulated as
\begin{align}
f\left(\boldsymbol X_t+\eta_t\boldsymbol\Delta\right)
\approx
f(\boldsymbol X_t)
+
\eta_t
\left\langle
\nabla f(\boldsymbol X_t),\boldsymbol\Delta
\right\rangle
+
\frac{\eta_t^2}{2}
\left\langle
\boldsymbol\Delta,\mathcal B_t(\boldsymbol\Delta)
\right\rangle .
\label{eq:second_approx_surrogate}
\end{align}
Here, \(\eta_t\) is the step size and
\(\boldsymbol\Delta\in\mathbb R^{m\times n}\) is the update direction. The corresponding curvature-aware update direction can be obtained by
minimizing \eqref{eq:second_approx_surrogate} with respect to \(\boldsymbol\Delta\):
\begin{align}
\boldsymbol\Delta_t
:=
\argmin_{\boldsymbol\Delta\in\mathbb R^{m\times n}}
\left\langle
\nabla f(\boldsymbol X_t),\boldsymbol\Delta
\right\rangle
+
\frac{\eta_t}{2}
\left\langle
\boldsymbol\Delta,\mathcal B_t(\boldsymbol\Delta)
\right\rangle .
\label{eq:update_dir_surrogate}
\end{align}
To understand the curvature-aware geometry characterized by \(\mathcal B_t\), we define the preconditioned 
update direction as
\begin{align}\label{curva-map}
\widetilde{\boldsymbol\Delta}
:=
\mathcal B_t^{1/2}(\boldsymbol\Delta),    
\end{align}
where \(\mathcal B_t^{1/2}\) denotes the positive definite square root of
\(\mathcal B_t\). Substituting the preconditioned variable into the quadratic term of \eqref{eq:update_dir_surrogate} yields
\[
\left\langle
\boldsymbol\Delta,\mathcal B_t(\boldsymbol\Delta)
\right\rangle
=
\left\|
\widetilde{\boldsymbol\Delta}
\right\|_F^2.
\]
Therefore, in the preconditioned space induced by
\(\mathcal B_t^{1/2}\), the quadratic term in above becomes isotropic. Equivalently, the anisotropic geometry in the original Euclidean space is transformed into an approximately isotropic geometry in the preconditioned space. Thus, the local model in
\eqref{eq:second_approx_surrogate} can be regarded as solving an optimization problem in the preconditioned space as
\begin{align}
\hat{\boldsymbol \Delta}_t: & = \argmin_{\widetilde{\boldsymbol \Delta} \in \mathbb R^{m\times n} }
\left\langle
\mathcal B_t^{-1/2}
\left(
\nabla f(\boldsymbol X_t)
\right),
\widetilde{\boldsymbol\Delta}
\right\rangle
+
\frac{\eta_t}{2}
\left\|
\widetilde{\boldsymbol\Delta}
\right\|_F^2 ,
\label{eq:whitened_local_model}
\end{align}
where the preconditioned gradient $ \mathcal B_t^{-1/2}
\left(
\nabla f(\boldsymbol X_t)
\right) $ is the gradient in the preconditioned space. The closed form solution of \eqref{eq:whitened_local_model} is $ \hat{\boldsymbol \Delta}_t: = -\frac{1}{\eta_t}\mathcal B_t^{-1/2}
\left(
\nabla f(\boldsymbol X_t)
\right)   $, which can be regarded as the gradient descent step in preconditioned space with step size $\eta_t$.

Although preconditioning can alleviate the curvature anisotropy, the direct gradient descent in preconditioned space has not explicitly considered the gradient anisotropy of $\mathcal B_t^{-1/2}
\left(
\nabla f(\boldsymbol X_t)
\right)$.  To further reduce gradient anisotropy in the preconditioned space, we apply idea of orthogonalization in Muon to the preconditioned gradient as
\begin{align}
\widetilde{\boldsymbol \Delta}_t:  = \operatorname{Orth} \left( \mathcal B_t^{-1/2}
\left(
\nabla f(\boldsymbol X_t)
\right) \right).
\end{align}

Finally, we can map the above updated direction in preconditioned space back to the original space based on \eqref{curva-map} to obtain the final update direction as
\begin{align}\label{final-update}
\boldsymbol{\Delta}_t = \mathcal B_t^{-1/2}\left(  \widetilde{\boldsymbol \Delta}_t  \right) =  \mathcal B_t^{-1/2}\left(   \operatorname{Orth} \left( \mathcal B_t^{-1/2}
\left(
\nabla f(\boldsymbol X_t)
\right) \right)   \right). 
\end{align}

In summary, preconditioned Muon first applies curvature-aware preconditioning to reduce curvature anisotropy, then orthogonalizes the preconditioned gradient in the preconditioned space to reduce gradient anisotropy, and finally maps the resulting direction back to the original parameter space to obtain the final update direction.

\subsection{Lightweight Diagonal Preconditioner}
Many existing optimizers can be viewed as using a preconditioning operator \(\boldsymbol{\mathcal B}: \mathbb R^{m \times n} \rightarrow \mathbb R^{m \times n}\) to incorporate curvature information. Examples include diagonal preconditioners built from elementwise gradient second moments \cite{kingma2015adam,loshchilov2019decoupled},
preconditioners based on Fisher information
\cite{amari1998natural,martens2015optimizing},
preconditioners based on Gauss–Newton matrices
\cite{martens2010deep,botev2017practical}, and structured preconditioners built from Kronecker factors approximation \cite{gupta2018shampoo,vyas2025soap}.
These methods differ in the trade-off between curvature expressiveness, computational and memory efficiency. Elementwise adaptive methods have low computational costs, but they rescale each parameter entry independently and do not explicitly use the row and column structure of matrix parameters \cite{gupta2018shampoo}. In contrast, Fisher, Gauss–Newton, and Shampoo methods capture richer curvature or matrix structure, but usually require dense blocks, Kronecker factors, or separate preconditioners for different tensor modes, together with matrix operations such as inversion or eigendecomposition.

We therefore use a pair of diagonal preconditioners, one on each side of the matrix. This choice preserves the row and column structure of matrix parameters without storing dense curvature factors or performing expensive matrix operations. Given the gradient
\(\boldsymbol G_t \in \mathbb R^{m \times n}\), we maintain exponential moving averages of the squared norms of its rows and columns:
\begin{align}
(\boldsymbol l_t)_i
&= \mu_2(\boldsymbol l_{t-1})_i
+ (1-\mu_2)\left\|(\boldsymbol G_t)_{i:}\right\|_2^2,
\quad \forall i \in [m],
\nonumber \\
(\boldsymbol r_t)_j
&= \mu_2(\boldsymbol r_{t-1})_j
+ (1-\mu_2)\left\|(\boldsymbol G_t)_{:j}\right\|_2^2,
\quad \forall j \in [n],
\nonumber \\
\boldsymbol P_t
&= \operatorname{Diag}\left((\boldsymbol l_t+\epsilon)^{-1/8}\right),
\quad
\boldsymbol Q_t
= \operatorname{Diag}\left((\boldsymbol r_t+\epsilon)^{-1/8}\right),
\label{eq:dp_preconditioner}
\end{align}
where \(\epsilon>0\) is a small damping constant for numerical stability. We then define the preconditioning operator by applying $\boldsymbol P_t$ and $\boldsymbol Q_t$ on the left and right as
\begin{equation}
\boldsymbol{\mathcal B}_t(\boldsymbol G)
:= \boldsymbol P_t \boldsymbol G \boldsymbol Q_t,
\quad \forall \boldsymbol G \in \mathbb R^{m \times n}.
\label{eq:two_sided_diagonal_preconditioner}
\end{equation}
Integrating the above preconditioner into Muon as \eqref{final-update} will obtain complete MALT update in Algorithm \ref{algo-1}. The motivation for this preconditioner is fourfold.

\begin{itemize}
    \item \textbf{Curvature-aware adaptive scaling.} The vectors \(\boldsymbol l_t\) and \(\boldsymbol r_t\) track EMAs of squared gradient norms across rows and columns , i.e., the diagonal parts of \(\boldsymbol G_t\boldsymbol G_t^\top\) and \(\boldsymbol G_t^\top\boldsymbol G_t\) over time. 
    Although they do not estimate the Hessian, they provide a lightweight approximation for curvature geometry along row and column directions of loss. The intuition is that large  
    entries in \(\boldsymbol l_t\) or \(\boldsymbol r_t\) suggest stronger local sensitivity or variation of the loss along the corresponding rows or column directions. \(\boldsymbol P_t\) 
    and \(\boldsymbol Q_t\) then suppress updates along such rows or columns and allow relatively larger updates along rows or columns with smaller accumulated squared gradient norm, yielding a lightweight curvature-aware scaling mechanism.

    \item \textbf{Matrix structure awareness.}
    Elementwise adaptive methods assign a separate scaling factor to every
parameter entry. In contrast, MALT scales a matrix
$\boldsymbol Z\in\mathbb R^{m\times n}$ according to
$
  \left(\boldsymbol P_t\boldsymbol Z\boldsymbol Q_t\right)_{ij}
  =
  \left(\boldsymbol P_t\right)_{ii}
  \boldsymbol Z_{ij}
  \left(\boldsymbol Q_t\right)_{jj}.$
Thus, all entries in the same row share the factor
$\left(\boldsymbol P_t\right)_{ii}$, while all entries in the same column
share the factor $\left(\boldsymbol Q_t\right)_{jj}$.
Each entry is therefore scaled using information from both its row and its
column, rather than from that entry alone.
    \item \textbf{Memory efficiency.}
    For each matrix parameter in \(\mathbb R^{m \times n}\), MALT only stores
    two vectors, \(\boldsymbol l_t \in \mathbb R^m\) and
    \(\boldsymbol r_t \in \mathbb R^n\), resulting in
    \(\mathcal O(m+n)\) additional optimizer state memory. This overhead is
    negligible compared with the \(\mathcal O(mn)\) memory required for the
    first order momentum and is substantially lower than storing dense matrix preconditioning factors.

    \item \textbf{Structured curvature approximation.} Existing studies \cite{zhang2024transformers,zhang2025adammini,dong2025towards} have shown that the Hessian of neural networks and Transformers often has a near block diagonal structure. This suggests that useful curvature information can be captured without forming a full dense Hessian. Motivated by this, MALT uses separable row and column gradient norms as a lightweight structured approximation to local curvature, reasonably providing a middle ground between the elementwise scaling of Adam and expensive dense matrix preconditioning.
\end{itemize}

\subsection{Relation to Existing Optimizers}
This section discusses the relation between the proposed MALT and existing optimizers.

\noindent\textbf{Relation to Adafactor.}
Adafactor \cite{shazeer2018adafactor} reduces memory usage by storing only the row and column sums of the moving average of squared gradients. It uses their normalized outer product to approximate the full elementwise second moment estimate. MALT shares a similar idea but uses the row and column sums for a different purpose. It constructs diagonal preconditioners on the left and right of the momentum matrix before orthogonalization, rather than using them to approximate the full elementwise second moment estimate.

\noindent\textbf{Relation to Shampoo.} Given the gradient $\boldsymbol G_t  \in \mathbb R^{m \times n}$, Shampoo  \cite{gupta2018shampoo} maintains dense accumulators for its rows and columns and defines the corresponding preconditioners as
\begin{align}\label{shampoo}
\boldsymbol L_t  &= \mu_2 \boldsymbol L_{t-1} + \left( 1 - \mu_2 \right) \boldsymbol G_t\boldsymbol G_t^\top \nonumber \\
\boldsymbol R_t &= \mu_2 \boldsymbol R_{t-1} + \left( 1- \mu_2 \right) \boldsymbol G_t^\top \boldsymbol G_t \nonumber \\
 \boldsymbol P_t & = \left( \boldsymbol L_t + \epsilon \boldsymbol I_m \right)^{-1/8}, \quad  \boldsymbol Q_t = \left( \boldsymbol R_t + \epsilon \boldsymbol I_n \right)^{-1/8} \nonumber \\
 \boldsymbol W_{t+1} & = \boldsymbol W_t - \eta_t \boldsymbol P_t \left(  \boldsymbol P_t \boldsymbol G_t \boldsymbol Q_t \right) \boldsymbol Q_t.
\end{align}
Thus, Shampoo essentially firstly takes the gradient descent step in preconditioned space under full dense matrix preconditioners $\boldsymbol P_t, \boldsymbol Q_t$ and then maps the update direction back to the original space. The preconditioners in \eqref{eq:dp_preconditioner} retain this left and right structure but replace the dense matrices of Shampoo with diagonal matrices. The inverse eighth power in \eqref{eq:dp_preconditioner} is adopted from the corresponding Shampoo factors in \eqref{shampoo}.


\noindent\textbf{Relation to SOAP.}
SOAP \cite{vyas2025soap} combines Shampoo preconditioning with Adam in the
eigenbasis of the Shampoo preconditioners. Let
\begin{equation}
    \boldsymbol L_t
    =
    \boldsymbol U_t \boldsymbol \Lambda_t \boldsymbol U_t^\top,
    \qquad
    \boldsymbol R_t
    =
    \boldsymbol V_t \boldsymbol \Gamma_t \boldsymbol V_t^\top,
    \qquad
    \widehat{\boldsymbol G}_t
    =
    \boldsymbol U_t^\top \boldsymbol G_t \boldsymbol V_t ,
\end{equation}
where
$\boldsymbol U_t\in\mathbb O^{m\times m}$
and
$\boldsymbol V_t\in\mathbb O^{n\times n}$
contain the eigenvectors of
$\boldsymbol L_t$ and $\boldsymbol R_t$.
SOAP applies Adam to the rotated gradient
$\widehat{\boldsymbol G}_t$ and maps the resulting direction back to the
original parameter space:
\begin{equation}
\begin{aligned}
    \widehat{\boldsymbol m}_t
    &=
    \beta_1\widehat{\boldsymbol m}_{t-1}
    +(1-\beta_1)\widehat{\boldsymbol G}_t,\\
    \widehat{\boldsymbol v}_t
    &=
    \beta_2\widehat{\boldsymbol v}_{t-1}
    +(1-\beta_2)
    \widehat{\boldsymbol G}_t\odot\widehat{\boldsymbol G}_t,\\
    \widehat{\boldsymbol \Delta}_t
    &=
    \frac{\widehat{\boldsymbol m}_t}
    {\sqrt{\widehat{\boldsymbol v}_t}+\epsilon},
    \qquad
    \boldsymbol W_{t+1}
    =
    \boldsymbol W_t
    -
    \eta_t
    \boldsymbol U_t
    \widehat{\boldsymbol \Delta}_t
    \boldsymbol V_t^\top .
\end{aligned}
\end{equation}
Hence, SOAP can be viewed as applying Adam in the eigenspace induced by the Shampoo preconditioners.

One may try to combine SOAP with Muon by applying orthogonalization in the
same eigenspace:
\begin{equation}
    \widehat{\boldsymbol G}_t
    =
    \boldsymbol U_t^\top\boldsymbol G_t\boldsymbol V_t,
    \qquad
    \widehat{\boldsymbol O}_t
    =
    \operatorname{Orth}(\widehat{\boldsymbol G}_t),
    \qquad
    \boldsymbol \Delta_t
    =
    \boldsymbol U_t
    \widehat{\boldsymbol O}_t
    \boldsymbol V_t^\top .
\end{equation}
This construction, however, gives exactly the same direction as applying
Muon in the original parameter space. The polar factor is equivariant under
orthogonal transformations:
\begin{equation}
    \operatorname{Orth}
    \left(
    \boldsymbol U_t^\top
    \boldsymbol G_t
    \boldsymbol V_t
    \right)
    =
    \boldsymbol U_t^\top
    \operatorname{Orth}(\boldsymbol G_t)
    \boldsymbol V_t .
\end{equation}
It follows that
\begin{equation}
\begin{aligned}
    \boldsymbol \Delta_t
    &=
    \boldsymbol U_t
    \operatorname{Orth}
    \left(
    \boldsymbol U_t^\top
    \boldsymbol G_t
    \boldsymbol V_t
    \right)
    \boldsymbol V_t^\top \\
    &=
    \operatorname{Orth}(\boldsymbol G_t).
\end{aligned}
\end{equation}
Thus, rotating the gradient into the Shampoo eigenbasis before
orthogonalization does not change the Muon direction. SOAP behaves
differently because Adam is not invariant under orthogonal transformation: its elementwise second moment depends on the basis in which it
is computed.

\noindent\textbf{Relation to Mousse and FISMO.}
Independent works Mousse \cite{zhang2026mousse} and FISMO \cite{xu2026fismo} are the most closely related curvature-aware Muon variants.
They share the same goal as MALT:  introducing preconditioning to address curvature anisotropy issue of vanilla Muon. There are three main key differences between these two works and ours. 
\begin{enumerate}

\item FISMO and Mousse use denser representations of curvature.
FISMO describes the local geometry using a Fisher metric with a Kronecker
structure, whereas Mousse uses dense Shampoo factors to construct a
preconditioned space. For a matrix parameter in
$\mathbb R^{m\times n}$, both methods maintain matrices on the left and
right, requiring $\mathcal O(m^2+n^2)$ additional optimizer state.

\item These dense factors also require additionally expensive matrix computations.
FISMO computes matrix inverses or inverse roots to obtain its left and right
preconditioners, while Mousse periodically performs eigendecompositions of
the Shampoo factors.

\item Mousse updates the eigenspaces of its preconditioner periodically
rather than at every step. Between two consecutive updates, the same
preconditioner are reused. The preconditioner may therefore respond with a
delay when the local curvature changes rapidly during training, which can result in unstable training as found in existing work \cite{frans2026stable}. 

\end{enumerate}

In contrast, MALT uses diagonal preconditioners on the left and right and
updates them at every step. This design may capture less curvature information
than dense preconditioning, but requires substantially less memory and
computation. MALT therefore remains close to Muon in cost while still
adapting the update to curvature.

\subsection{ Norm Grafting in MALT and Adaptive MALT }
In MALT, we obtain the update as $\boldsymbol D_t = \boldsymbol L_t \boldsymbol O_t \boldsymbol R_t$ in the 10th line of Algorithm \ref{algo-1}. The eighth square root of $\boldsymbol l_t, \boldsymbol R_t$ can make the update too aggressive if some rows or columns have very small norms and the damping parameter $\epsilon$ is too small. Thus, when we try to use $\boldsymbol D_t$ as an update directly, we find that we have to use a very small learning rate to prevent divergence, which results in slow convergence. Thus, we use the idea of grafting \cite{agarwal2020disentangling,anil2020scalable} to decouple the update direction and update magnitude. 

 \begin{algorithm}[htbp]
\caption{Adaptive MALT (MALTER) Update}
\label{algo-2}
\begin{algorithmic}[1]
\Require learning rate $\eta$, momentum $\mu_1,\mu_2 \in [0,1)$, batch size $B$ and damping constant $\epsilon$.
\INITIALIZE  $\boldsymbol W_0\in\mathbb R^{m\times n}$, $\boldsymbol M_0=\boldsymbol 0_{m\times n}$, $\boldsymbol l_0=\boldsymbol 0_m$, $\boldsymbol r_0=\boldsymbol 0_n$, and $\nu_0=0$.

\For{$t=1,\ldots,T$}
\State Sample a minibatch $\mathcal B_t$ with $|\mathcal B_t|= B$ and compute the stochastic gradient $\boldsymbol G_t=\frac{1}B\sum_{i\in\mathcal B_t}\nabla f_i(\boldsymbol W_{t-1};\xi_i)$.

\State $\boldsymbol M_t\gets\mu_1\boldsymbol M_{t-1}+(1-\mu_1)\boldsymbol G_t$.

\State $\boldsymbol l_t\gets\mu_2\boldsymbol l_{t-1}+(1-\mu_2)\left[|\boldsymbol G_t[1,:]|_2^2;\cdots;|\boldsymbol G_t[m,:]|_2^2\right]$.

\State $\boldsymbol r_t\gets\mu_2\boldsymbol r_{t-1}+(1-\mu_2)\left[|\boldsymbol G_t[:,1]|_2^2;\cdots;|\boldsymbol G_t[:,n]|_2^2\right]$.

\State $\boldsymbol L_t\gets\operatorname{Diag}\left((\boldsymbol l_t+\epsilon)^{-1/8}\right)$ and $\boldsymbol R_t\gets\operatorname{Diag}\left((\boldsymbol r_t+\epsilon)^{-1/8}\right)$.

\State $\widetilde {\boldsymbol M}_t \gets \boldsymbol L_t \boldsymbol M_t \boldsymbol R_t$
\State $\boldsymbol O_t \gets \Orth(\widetilde {\boldsymbol M}_t)$

\State $\boldsymbol D_t \gets \boldsymbol L_t \boldsymbol O_t \boldsymbol R_t$

\State $v_t\gets\mu_2 v_{t-1}+(1-\mu_2) \left\|\boldsymbol L_t \boldsymbol G_t \boldsymbol R_t\right\|_F^2$.

\State $\alpha_t\gets\eta\frac{\sqrt{1-\mu_2^t}}{1-\mu_1^t+\epsilon}\frac{\left\|\widetilde{\boldsymbol M}_t\right\|_F}{\sqrt{\nu_t}+\epsilon}$.

\State $\boldsymbol W_{t}\gets \boldsymbol W_{t-1} -\eta \alpha_t \frac{\|\boldsymbol O_t\|_F}{\|\boldsymbol D_t\|_F+\epsilon}\boldsymbol D_t$
\EndFor

\State \Return $\boldsymbol W_T$.
\end{algorithmic}
\end{algorithm}

We normalize the preconditioned update $\boldsymbol D_t$ to obtain the update direction with unit Frobenius norm. For the updating magnitude and we utilize the Frobenius norm of the orthogonal update $\boldsymbol O_t$ as the grafting norm. This grafting aims to make the update magnitude match that of the orthogonal update in the preconditioned space.

But the update magnitude obtained by grafting in the 11th line of Algorithm \ref{algo-1} cannot make the update magnitude adapt to the stochastic gradient noise. Motivated by existing work \cite{frans2025really}, which has shown that noise adaptation is as important as the update direction for the optimizers. The optimizers with noise adaptation consistently outperform their counterparts without noise adaptation. For example, Adam is better than Signum \cite{bernstein2018signsgd}, SOAP is better than SPlus \cite{frans2026stable}, and AdaMuon \cite{si2025adamuon} is better than Muon. We propose the noise-adaptive version of MALT to enhance the performance of MALT as Algorithm \ref{algo-2}.

The noise-adaptive stepsize in MALTER follows the same norm-based Adam-type principle as NAMO \cite{zhang2026namo}, but it is applied in the preconditioned space induced by MALT. The key idea of NAMO is to keep orthogonal update direction of Muon, while using an Adam-type scalar to adjust its magnitude according to the noise level of stochastic gradients. MALTER adopts the same idea after diagonal precondition. Specifically, it forms the preconditioned momentum $\widetilde {\boldsymbol M}_t= \boldsymbol L_t \boldsymbol M_t \boldsymbol R_t$ and maintains a scalar estimate $\nu_t=\mu_2\nu_{t-1}+(1-\mu_2)\left\|\boldsymbol L_t \boldsymbol G_t \boldsymbol R_t\right\|_F^2$. The adaptive factor
\begin{align}
 \alpha_t = 
\frac{\sqrt{1-\mu_2^t}}{1-\mu_1^t+\epsilon}
\frac{\left\|\widetilde {\boldsymbol M}_t \right\|_F}{\sqrt{\nu_t}+\epsilon}  
\end{align}
therefore approximately measures the signal-to-noise ratio of stochastic preconditioned gradient. Thus, MALTER preserves Adam-style noise adaptation in NAMO, while adapting it to the curvature-aware precondition of MALT. This makes the update scale responsive to stochastic gradient fluctuations without discarding the preconditioned orthogonalization mechanism used to address curvature and gradient anisotropy. In addition, we can observe that the memory and computational overhead of MALTER is negligible compared with MALT and Muon.

\section{Convergence Analysis of MALT}
This section gives the convergence guarantee of MALT. Before providing the convergence result of MALT, we need the following standard assumptions for the loss function \eqref{loss-fun}.
\begin{assumption}\label{assum-1}(Lower boundedness). Loss function $f$ admits a finite lower bound, i.e.
\begin{align}
\inf_{\boldsymbol W \in \mathbb R^{m\times n}} f\left( \boldsymbol W\right): = f^\star >-\infty.
\end{align}

\end{assumption}
\begin{assumption}\label{assum-2}(L-smoothness)
 The loss function $f: \mathbb R^{m \times n} \rightarrow \mathbb R$ is differentiable and there exists a constant $L>0$ such that
 \begin{align}
f\left( \boldsymbol W_2\right) \leq f\left( \boldsymbol W_1\right) + \langle\nabla f \left( \boldsymbol W_1 \right) , \boldsymbol W_2 - \boldsymbol W_1\rangle + \frac{L}{2} \left\| \boldsymbol W_2 - \boldsymbol W_1 \right\|_F^2.
 \end{align}
\end{assumption}

\begin{assumption}\label{assum-3}(Bounded variance).
Let $\xi$ be a random variable that is independent of $\boldsymbol W \in \mathbb R^{m \times n}$, we assume $\nabla f \left( \boldsymbol W; \xi\right)$ is an unbiased stochastic estimator of true gradient $\nabla f\left( \boldsymbol W \right)$ and have a bounded variance, i.e.
\begin{align}
\mathbb E \left[  \nabla f \left( \boldsymbol W;\xi\right) \right] = \nabla f\left( \boldsymbol W \right),\quad \quad \mathbb E \left[ \left\|   \nabla f \left( \boldsymbol W; \xi\right) -  \nabla f\left( \boldsymbol W \right) \right\|_F^2 \right] \leq \sigma^2.
\end{align}
\end{assumption}

Based on the above assumptions, we have the convergence guarantee for the proposed MALT in the following theorem.
\begin{theorem}\label{thm-1}
Without loss of  generality, we assume $m\geq n$, if Assumptions \ref{assum-1}-\ref{assum-3} and following conditions
\begin{align}
\eta & = \frac{\epsilon^{\frac 5 8}}{n^{\frac 3 4}} \sqrt{ \frac{2\Delta_{0,1}\left(1- \mu_1 \right)}{5LT} },\quad 1-\mu_1= \min \left\{  \epsilon^{\frac 5 8}  \sqrt{ \frac{ 5 LB \Delta_{0,1}}{2T n^{\frac 1 2}\sigma^2}} , 1 \right\}
\end{align}
are satisfied, then $\forall \epsilon \in (0,1)$ the proposed MALT in Algorithm \ref{algo-1} has
\begin{align}\label{final-result}
\min_{t = 1,\cdots, T}    \min \left\{ \mathbb E \left[ \left\| \nabla f\left( \boldsymbol W_{t-1} \right) \right\|_F^{\frac 1 2}\right] , \frac{ \mathbb E^2 \left[ \left\| \nabla f\left( \boldsymbol W_{t-1} \right) \right\|_F^{\frac 1 2}\right]}{  1 + 2n^{\frac 1 4} \sqrt{\frac{\eta L}{1-\mu_2}} + \frac{\sqrt{\sigma}}{B^{\frac 1 4}}  }\right\}  & \leq 16 \epsilon^{-\frac{19}{16}} n^{\frac 7 8} \sqrt[4]{ \frac{L\Delta_{0,1}\sigma^2}{TB}} + 8 \epsilon^{-\frac{17}{8}} n ^{\frac 5 4} \frac{\Delta_{0,2} \sigma}  {\sqrt{TLB\Delta_{0,1}}} \nonumber \\
& \quad + 31 \epsilon^{-\frac 7 8} n ^{\frac 3 4} \sqrt{\frac{L \Delta_{0,1}}{T}}  + \frac{8 \epsilon^{-\frac 3 2} n \Delta_{0,2}}{T},
\end{align}
where $\Delta_{0,1}: = f\left(\boldsymbol W_0 \right) - f^\star, \Delta_{0,2}: = \left\| \nabla  f\left(\boldsymbol W_0 \right)\right\|_F$.
\end{theorem}
Compared with standard Muon convergence bounds, the stationarity measure in Theorem \ref{thm-1} is different from the usual Euclidean gradient norm $\mathbb{E} \left[ \left\|\nabla f(\boldsymbol W_t)\right\|_F \right]$ \cite{nagashima2026improved,zhang2026namo} or nuclear-norm measure $\mathbb{E} \left[ \left\|\nabla f(\boldsymbol W_t)\right\|_* \right]$ \cite{shen2025convergence}. This difference comes from the stochastic diagonal preconditioners $\boldsymbol L_t$ and $\boldsymbol R_t$. In the proof, the descent lemma involves controlling the preconditioned gradient norm $\left\|\boldsymbol L_t\nabla f(\boldsymbol W_{t-1})\boldsymbol R_t\right\|_F$, rather than directly through $\left\|\nabla f(\boldsymbol W_{t-1})\right\|_F$, which introduces a random denominator in the lower bound of the preconditioned gradient norm as in \eqref{lower-bound}. 

When we seek a high-accuracy solution such that
$
\mathbb{E}\left[\left\|\nabla f\left(\boldsymbol W_{t-1}\right)\right\|_F^{1/2}\right]
$
is sufficiently small, the convergence metric is dominated by
$
\mathbb{E}^2\left[\left|\nabla f(\boldsymbol W_{t-1})\right\|_F^{1/2}\right].
$
This implies that the complexity of finding a point satisfying $
\mathbb E^2\left[
\left\|\nabla f\left(\boldsymbol W_{t-1}\right)\right\|_F^{1/2}
\right]\le \delta$
is $
\mathcal O\left(
\frac{n^{7/2}L\Delta_{0,1}\sigma^2}{B}\delta^{-4}
\right)$. This matches the \(\delta^{-4}\) accuracy dependence of existing stochastic
non-convex analyses of Muon
\cite{shen2025convergence,nagashima2026improved}, while incurring a worse
polynomial dependence on the matrix dimension \(n\). However, since $
\mathbb{E}^2\left[ \left\|\nabla f(\boldsymbol W_{t-1})\right|_F^{1/2}\right] \leq
\mathbb{E}\left[ \left\|\nabla f(\boldsymbol W_{t-1})\right|_F\right]$, this convergence metric indicates that our convergence guarantee is weaker than those in existing results.

\section{Experiments}
\label{sec:experiments}

This section presents the performance of our proposed optimizers on large language model (LLM) pretraining tasks

\subsection{Settings}
\label{subsec:settings}

\noindent\textbf{Models and Dataset.}
We evaluate MALT and the baselines on GPT-2 pretraining. The model implementation is based on NanoGPT~\cite{karpathy2022nanogpt}. In order to validate the scalability of our proposed optimizers, we consider three model sizes: GPT-2 Small (124M), GPT-2 Medium (355M), and GPT-2 Large (774M). All experiments are conducted on the OpenWebText dataset~\cite{gokaslan2019openwebtext}, which contains approximately 9B training tokens and 4.4M validation tokens.

\noindent\textbf{Baselines and Parameter Grouping.} We compare our proposed MALT and MALTER with Muon~\cite{jordan2024muon} and AdamW~\cite{kingma2015adam}. Since Muon-style updates are designed for matrix-valued parameters, we apply Muon, MALT, and MALTER to matrix parameters only, and use AdamW to optimize the vectors, scalars, embedding layers, and output head layers. For a fair comparison, this parameter grouping remains the same among Muon-style optimizers.

\noindent\textbf{Training Protocols.} All experiments use a global batch size of 480 and a block size of 1024, corresponding to 491{,}520 tokens per optimizer step. We evaluate training and validation losses every 500 steps by averaging over 200 mini-batches. 

For a fair comparison, we first perform an extensive grid search over learning rates, as is shown in Table~\ref{tab:lr_grid}. For each optimizer, we train the model for 10K steps and select the best learning rate according to the minimum validation loss. The optimal learning rates are reported in Table~\ref{tab:optimal_lr}. We use a linear warmup schedule for the first 10\% of training steps to reach the stable learning rate. 

In order to further find out the final convergent performance, we continue to train the models using the checkpoint of each model under the optimal learning rate. GPT-2 Small, Medium, and Large are trained for 50K, 10K, and 20K steps in total, respectively. 

\noindent\textbf{Hardware.}
All pretraining experiments are conducted on a cloud GPU cluster with 4 H100 GPUs, and a workstation with 2 PRO6000 GPUs. A subset of the learning-rate sweep experiments for GPT-2 Small and Medium is conducted on PRO6000 GPUs. Since these runs are used only for loss comparison under the same training protocol, hardware differences do not affect the reported validation-loss values. All wall-clock time and memory measurements are conducted on a single H100 GPU.

\begin{table}[htbp]
\centering
\renewcommand{\arraystretch}{1.2}
\begin{tabular}{lcccc} 
\toprule
\multirow{2}{*}{\textbf{Model Size}} & \multicolumn{4}{c}{\textbf{Optimal Learning Rate}} \\
\cmidrule(lr){2-5}
& \textbf{AdamW} & \textbf{Muon} & \textbf{MALT} & \textbf{MALTER} \\
\midrule
\textbf{Small} (124M)  & $0.0013$ & $0.0013$ & $0.0013$ & $0.012$ \\
\textbf{Medium} (355M) & $0.0009$ & $0.0009$ & $0.0009$ & $0.009$ \\
\textbf{Large} (774M)  & $0.0009$ & $0.0007$ & $0.0007$ & $0.006$ \\
\bottomrule
\end{tabular}
\vspace{0.3cm}
\caption{Optimal learning rates found via grid search for each model scale and optimizer. (See Table~\ref{tab:additional_settings} for the configuration of other hyperparameters.)}
\label{tab:optimal_lr}
\end{table}

\begin{table}[htbp]
\centering
\begin{tabular}{lll} 
\toprule
\textbf{Model Size} & \textbf{Optimizer} & \textbf{Learning Rate Grid} \\
\midrule

\multirow{2}{*}{\textbf{Small} (124M)}  
& AdamW, Muon, MALT & \{0.0005, 0.0009, 0.0013, 0.0018, 0.0025, 0.0032\} \\
& MALTER     & \{0.007, 0.009, 0.012, 0.015, 0.018, 0.025\} \\
\midrule

\multirow{2}{*}{\textbf{Medium} (355M)} 
& AdamW, Muon, MALT & \{0.0005, 0.0009, 0.0013, 0.0018, 0.0025, 0.0032\} \\
& MALTER     & \{0.005, 0.006, 0.007, 0.009, 0.013, 0.015\} \\
\midrule

\multirow{2}{*}{\textbf{Large} (774M)}  
& AdamW, Muon, MALT & \{0.0003, 0.0005, 0.0006, 0.0007, 0.0009, 0.0013\} \\
& MALTER     & \{0.003, 0.005, 0.006, 0.007, 0.009, 0.013\} \\

\bottomrule
\end{tabular}
\vspace{0.3cm}
\caption{Full learning rate search grids for all model scales. AdamW, Muon, and MALT share the same search space, while MALTER is tuned on a distinct grid.}
\label{tab:lr_grid}
\end{table}

\subsection{Results of GPT-2 Pretraining}
\label{subsec:results_pretraining}
\begin{table}[htbp]
\centering
\begin{tabular}{ll cccc}
\toprule
\multirow{2}{*}{\textbf{Model Size}} & \multirow{2}{*}{\textbf{Metric}} & \multicolumn{4}{c}{\textbf{Optimizers}} \\
\cmidrule(lr){3-6}
& & \textbf{AdamW} & \textbf{Muon} & \textbf{MALT} & \textbf{MALTER} \\
\midrule

\multirow{2}{*}{\textbf{Small} (124M)}  
& Train Loss & $3.0467$ & $3.0298$ & $\mathbf{3.0193}$ & $\mathbf{{3.0063}}$ \\
& Val Loss   & $3.0540$ & $3.0472$ & $\mathbf{3.0358}$ & $\mathbf{3.0231}$ \\
\midrule

\multirow{2}{*}{\textbf{Medium} (355M)} 
& Train Loss & $2.9851$ & $2.9599$ & $\mathbf{2.9490}$ & $\mathbf{2.9282}$ \\
& Val Loss   & $2.9971$ & $2.9719$ & $\mathbf{2.9611}$ & $\mathbf{2.9442}$ \\
\midrule

\multirow{2}{*}{\textbf{Large} (774M)}  
& Train Loss & $3.2404$ & $2.7604$ & $\mathbf{2.7480}$ & $\mathbf{2.7442}$ \\
& Val Loss   & $3.2705$ & $2.7848$ & $\mathbf{2.7723}$ & $\mathbf{2.7684}$ \\

\bottomrule
\end{tabular}
\vspace{0.3cm}
\caption{Final training and validation loss achieved by different optimizers across varying model scales. The maximum steps of small, medium and large models are 50K, 10K and 20K, respectively.}
\label{tab:final_loss}
\end{table}

\noindent\textbf{MALT and MALTER improve GPT-2 pretraining.}
 Table~\ref{tab:final_loss} reports the final training and validation losses under the selected learning rates. Muon already improves over AdamW on validation loss, reducing the loss by $0.0068$, $0.0252$, and $0.4857$ on GPT-2 Small, Medium, and Large, respectively. MALT further improves over Muon, with additional validation-loss reductions of $0.0114$, $0.0108$, and $0.0125$ across the three model scales. MALTER achieves the largest improvement, further reducing the validation loss over Muon by $0.0241$, $0.0277$, and $0.0164$, and obtains the best validation loss in all settings. These results show that the proposed methods consistently strengthen Muon on GPT-2 pretraining while preserving the advantage of Muon over AdamW.

Figures~\ref{fig:small_loss}, \ref{fig:medium_loss} and \ref{fig:large_loss} show that MALT and MALTER outperform Muon and AdamW throughout the pretraining process. On all three models, their training and validation loss curves remain below AdamW and Muon. This indicates that the proposed MALT and MALTER improve not only the final loss but also the iteration efficiency. In particular, MALT and MALTER reach lower loss values with fewer optimization steps. Among the compared methods, MALTER achieves the best overall performance.

\begin{figure}[htbp]
\centering
\subfigure[Training Loss]{
\begin{minipage}[t]{0.5\linewidth}
\centering
\includegraphics[width=3.5in]{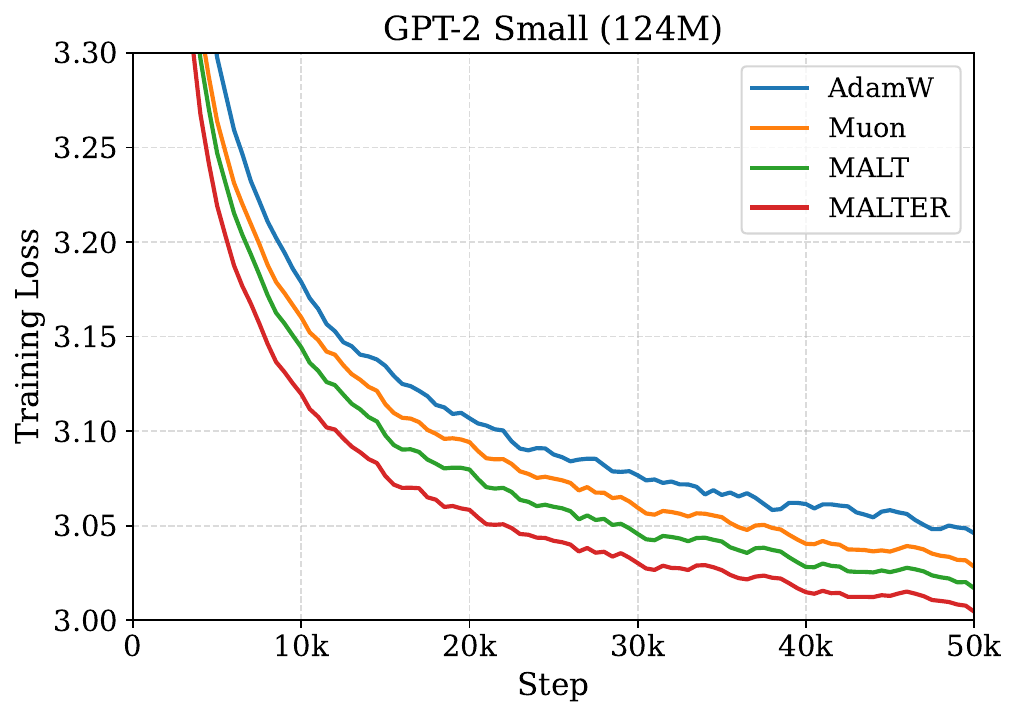}
\label{fig:small_tr}
\end{minipage}%
}%
\subfigure[Validation Loss]{
\begin{minipage}[t]{0.5\linewidth}
\centering
\includegraphics[width=3.5in]{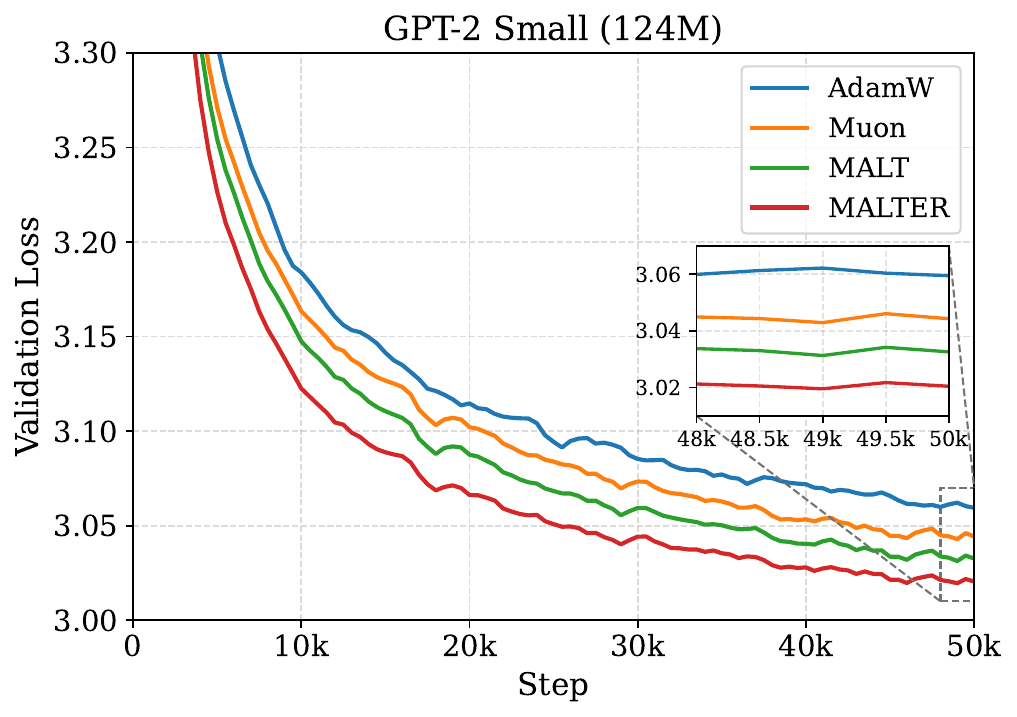}
\label{fig:small_val}
\end{minipage}%
}%
\centering
\caption{ GPT-2 small (124M) pretraining. The small model is trained under learning rates reported in Table~\ref{tab:optimal_lr} for 50K steps.}
\label{fig:small_loss}
\end{figure}

\begin{figure}[htbp]
\centering
\subfigure[Training Loss]{
\begin{minipage}[t]{0.5\linewidth}
\centering
\includegraphics[width=3.5in]{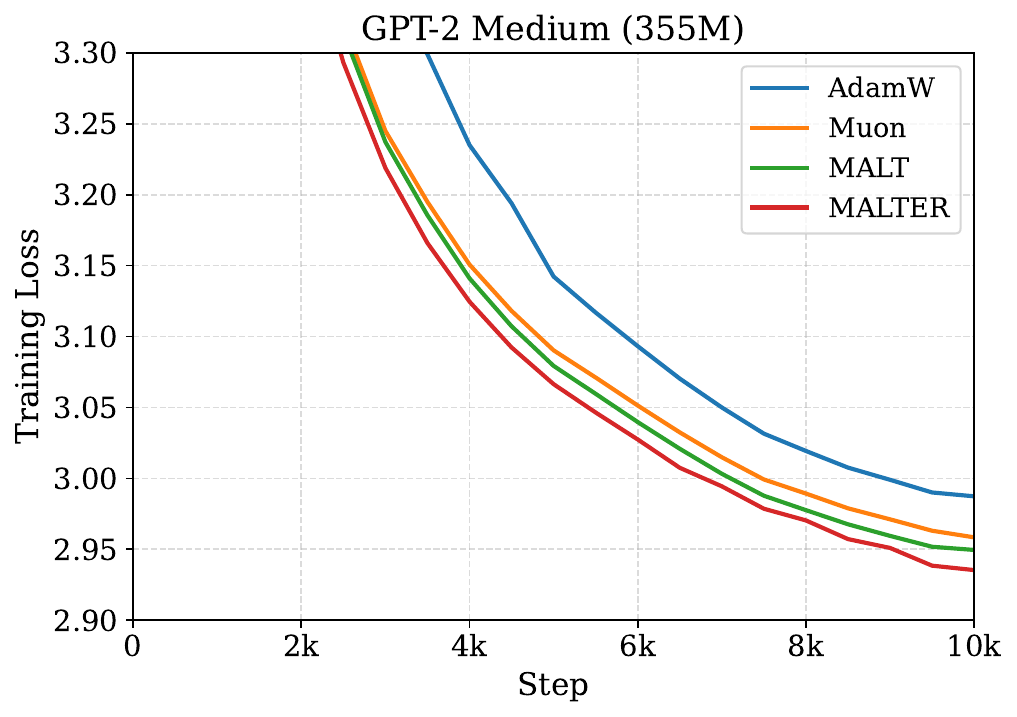}
\label{fig:medium_tr}
\end{minipage}%
}%
\subfigure[Validation Loss]{
\begin{minipage}[t]{0.5\linewidth}
\centering
\includegraphics[width=3.5in]{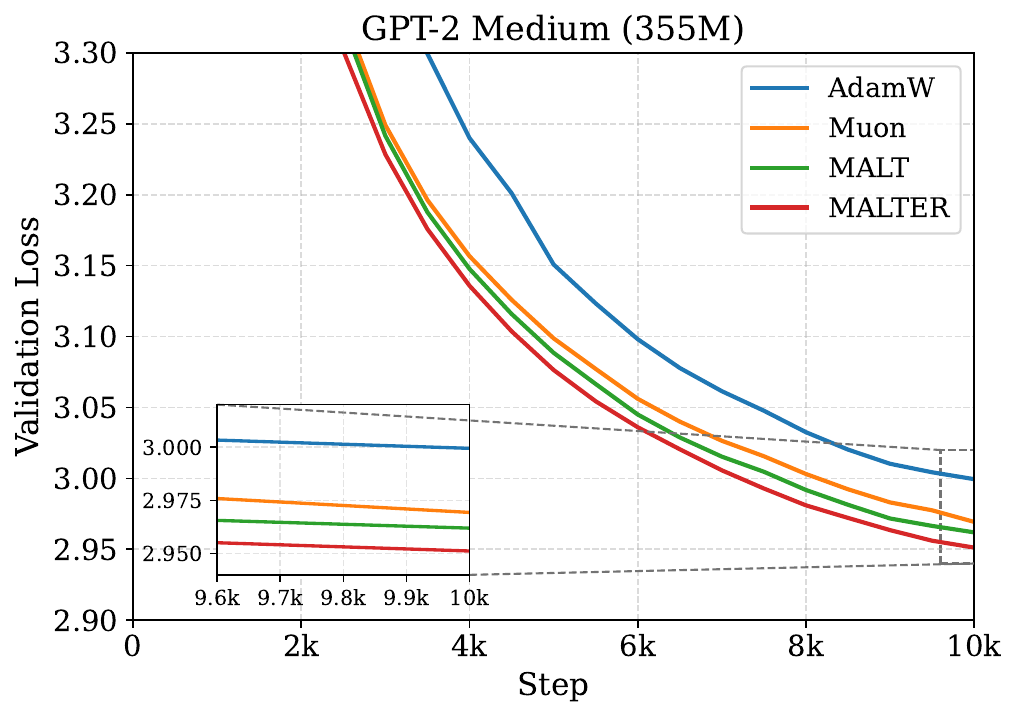}
\label{fig:medium_val}
\end{minipage}%
}%
\centering
\caption{ GPT-2 medium (355 M) pretraining. The medium model is trained under learning rates reported in Table~\ref{tab:optimal_lr} for 10K steps.}
\label{fig:medium_loss}
\end{figure}
\begin{figure}[htbp]
\centering
\subfigure[Training Loss]{
\begin{minipage}[t]{0.5\linewidth}
\centering
\includegraphics[width=3.5in]{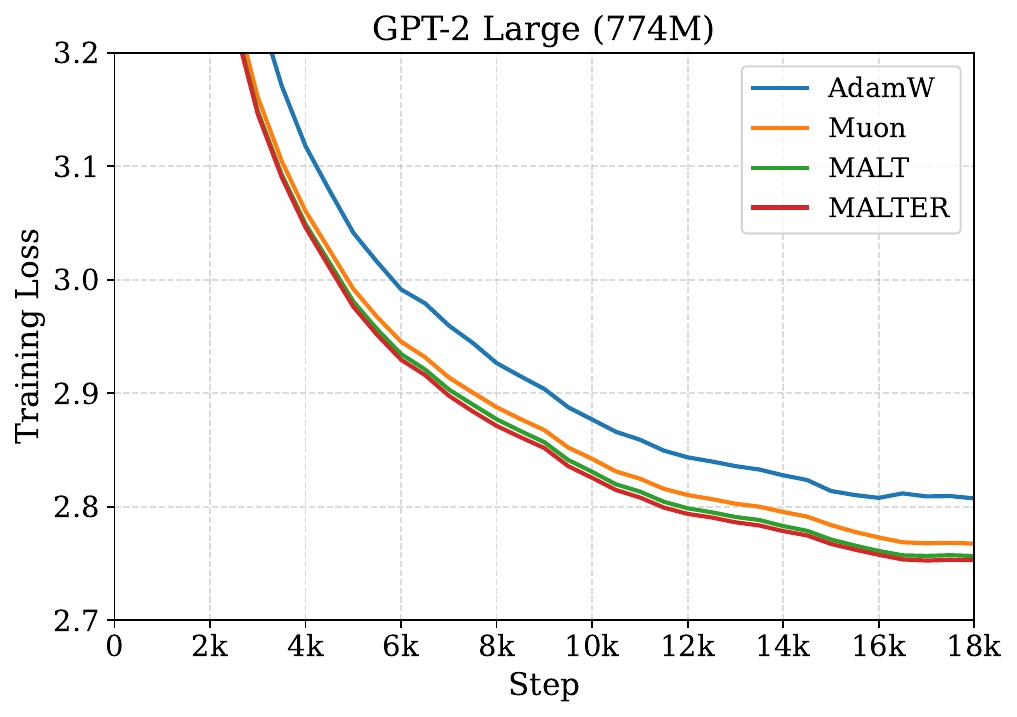}
\label{fig:large_tr}
\end{minipage}%
}%
\subfigure[Validation Loss]{
\begin{minipage}[t]{0.5\linewidth}
\centering
\includegraphics[width=3.5in]{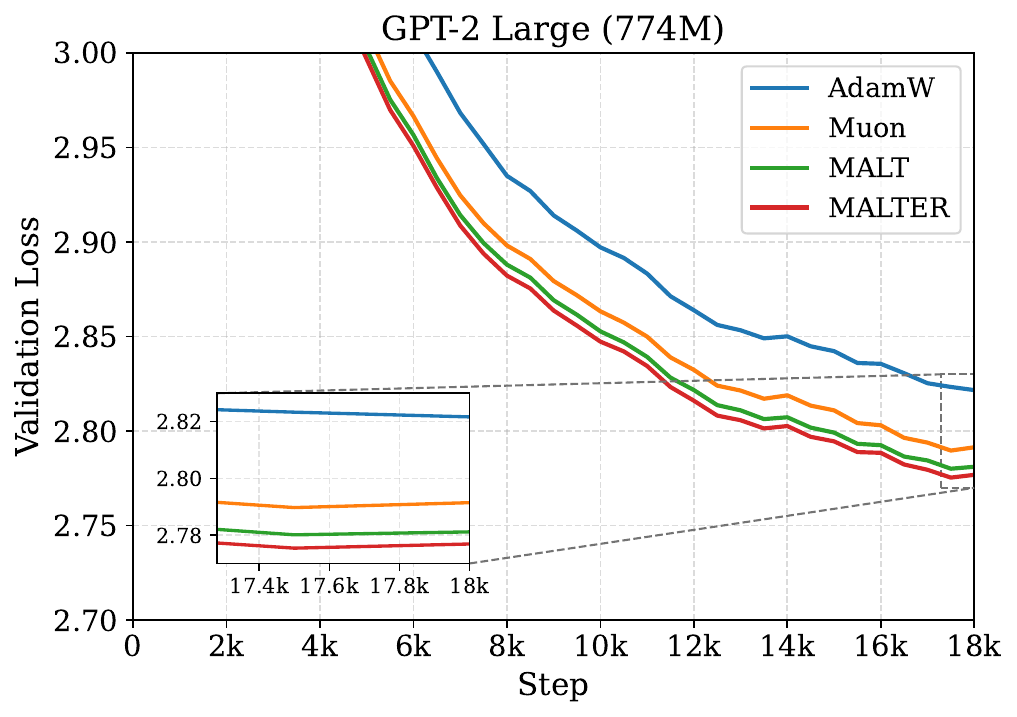}
\label{fig:large_val}
\end{minipage}%
}%
\centering
\caption{ GPT-2 large (774 M) pretraining. The large model is trained at learning rates reported in Table~\ref{tab:optimal_lr} for 20K steps, we plot the first 18K steps. \textit{(Remark: AdamW diverges in the last 500 steps, with final validation loss 3.2705.)}}
\label{fig:large_loss}
\end{figure}


\subsection{Performance Under Different Learning Rates}
\label{subsec:lr_sweep}

In this section, we report the results of the 10K-step learning-rate sweep described in Section~\ref{subsec:settings}.  Figure~\ref{fig:lr_sweep} shows the final validation loss across the learning-rate grids for training GPT-2 Small, Medium, and Large for 10K steps. These sweep experiments not only find the optimal learning rate of each optimizer, but also reflect the consistent performance gain of MALT and MALTER under different learning rates. As presented in Figure~\ref{fig:lr_sweep}, on each learning rate grid, MALT achieves a lower validation loss than Muon and AdamW. In particular, MALTER attains the best performance across all model scales. These results suggest that the gains of MALT and MALTER are not merely due to a favorable learning-rate choice.

\begin{figure}[htbp]
    \centering
    \includegraphics[width=\linewidth]{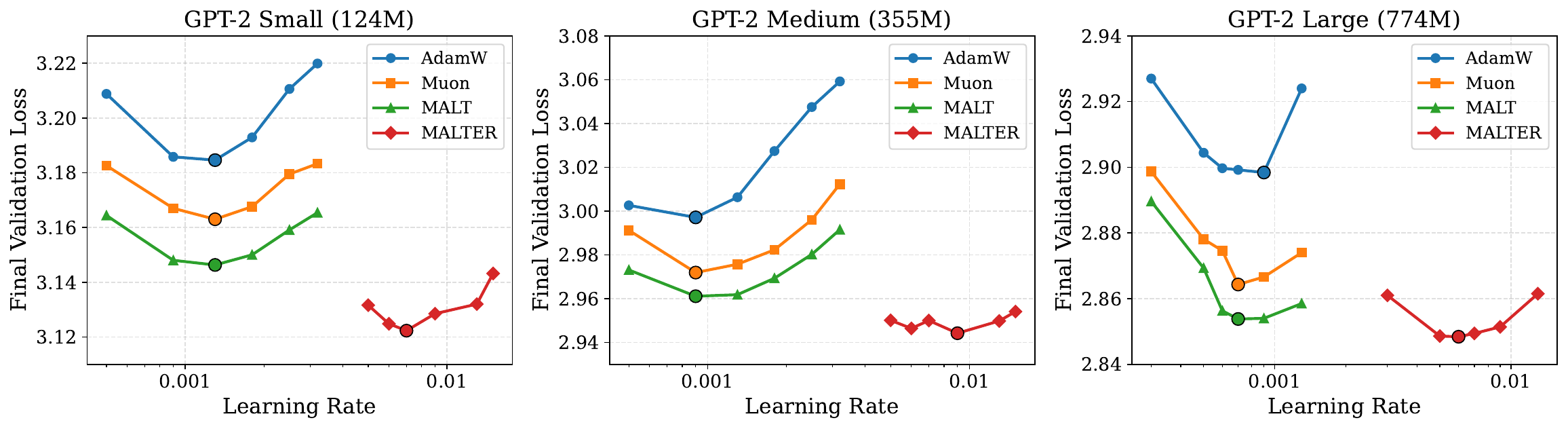}
    \caption{Final validation loss of pretraining GPT-2 Small, Medium and Large for 10K steps. 
    The learning-rate grids are shown in Table~\ref{tab:lr_grid}.
    The best point for each optimizer is marked by a dot.}
    \label{fig:lr_sweep}
\end{figure}

\subsection{Analysis of Diagonal Preconditioners}
\label{subsec:preconditioner_analysis}

\begin{figure*}[htb]
    \centering
    \includegraphics[width=0.98\textwidth]{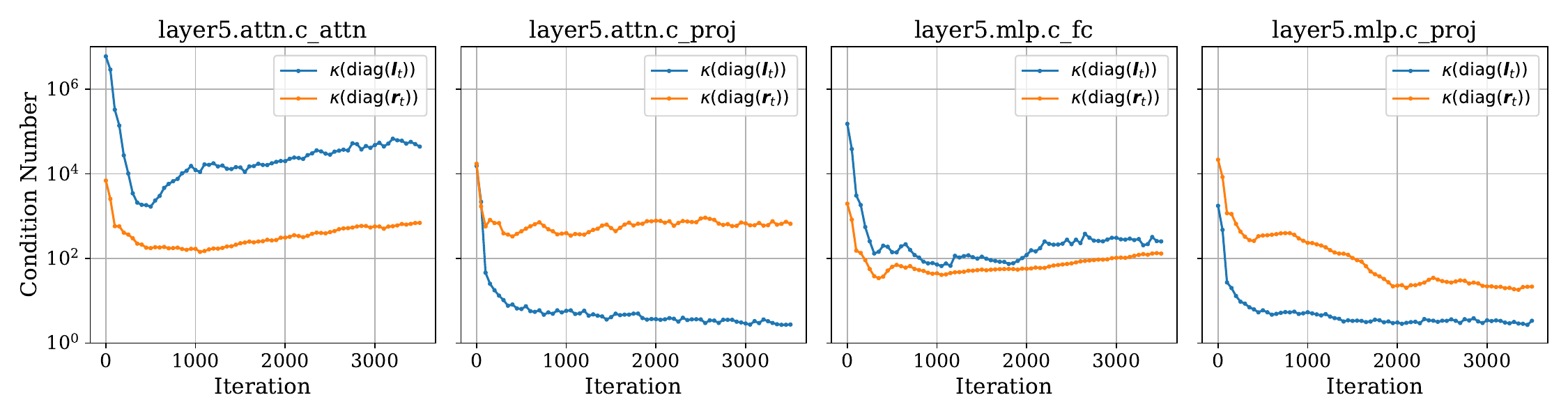}
    \includegraphics[width=0.98\textwidth]{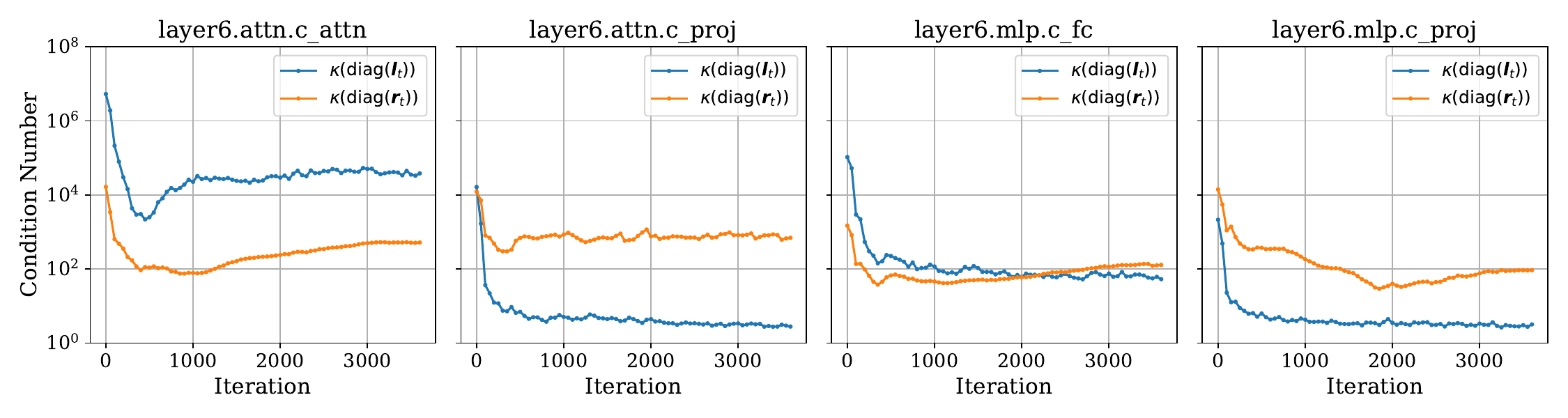}
    \includegraphics[width=0.98\textwidth]{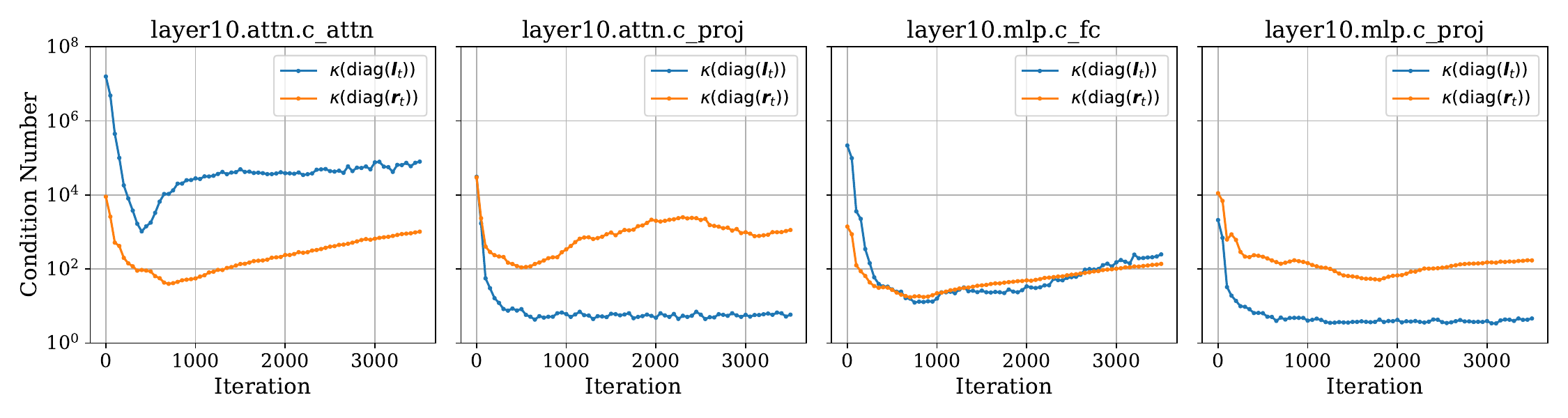}
    \caption{
    Condition numbers of the diagonal preconditioners in MALT for representative transformer layers.
    From top to bottom, we show Layer 0, Layer 5, and Layer 10.
    For each layer, we report the condition number of $\operatorname{diag}(\boldsymbol{l}_t)$ and
    $\operatorname{diag}(\boldsymbol{r}_t)$ for attention and MLP matrices.
    }
    \label{fig:diag_cond}
\end{figure*}

In this section, we further analyze the left and right diagonal preconditioners of MALT. To verify that these preconditioners are not trivial scalings, we train a GPT-2 medium model for 3500 steps using MALT, with a learning rate $0.0009$, and collect the $\boldsymbol{l}_t$ and $\boldsymbol{r}_t$ every 50 steps.

In Figure~\ref{fig:diag_cond}, we visualize the condition numbers of the left and right diagonal preconditioners for representative Transformer blocks.
Here, \texttt{attn.c\_attn} denotes the fused query-key-value projection,
\texttt{attn.c\_proj} denotes the output projection of the self-attention module,
\texttt{mlp.c\_fc} denotes the expansion layer in the feed-forward network, and
\texttt{mlp.c\_proj} denotes the corresponding down-projection layer. It can be seen that the condition numbers of preconditioners are often much larger than one across different layers and modules. 
This indicates that the diagonal preconditioners are non-trivial and capture substantial row- and column-wise anisotropy of the gradient during training.

\subsection{Time and Memory Overhead Analysis}
\label{subsec:overhead}
In this section, we demonstrate the memory and computation efficiency of MALT and MALTER. We conducted concrete experiments to compare the optimizer state memory and per-iteration wall-clock time of our proposed MALT and MALTER with AdamW and Muon. 

For a matrix parameter $W\in\mathbb{R}^{m\times n}$, AdamW stores two momentum matrices of size $mn$, while Muon stores one momentum matrix of size $mn$. Compared to Muon, MALT and APDMuon introduce mild additional memory overheads. Specifically, MALT introduces two additional vectors of sizes $m$ and $n$ for each matrix parameter. Except for two vectors, MALTER further introduces only one scalar state per matrix parameter. 


\begin{table}[htbp]
\centering
\label{tab:memory_complexity}
\begin{tabular}{lc}
\toprule
Optimizer & Optimizer State Complexity\\
\midrule
AdamW   & $2mn$ \\
Muon    & $mn$ \\
MALT  & $mn+m+n$ \\
MALTER & $mn+m+n+1$ \\
\bottomrule
\end{tabular}
\caption{
Optimizer-state memory complexity of AdamW, Muon, MALT, and MALTER. 
}
\end{table}

\begin{table}[h]
\centering
\label{tab:time_memory}
\begin{tabular}{lcc}
\toprule
Optimizer & Optimizer State Memory (MB) & Time per Iteration (ms) \\
\midrule
AdamW   & $2705.38$   & $2605.16$ \\
Muon    & $1553.38$   & $2637.62$ \\
MALT  & $1554.88$   & $2652.55$ \\
MALTER & $1554.88$   & $2666.22$ \\
\bottomrule
\end{tabular}
\caption{
Optimizer-state memory and per-iteration wall-clock time during the pretraining of GPT-2 Medium (355M) tested on a single H100 GPU. 
}
\label{tab:time_memory}
\end{table}

Apart from the memory efficiency, MALT and MALTER enjoy high computational efficiency. We measure the per-iteration wall-clock time and optimizer state memory of AdamW, Muon, MALT and MALTER. As shown in Table~\ref{tab:time_memory}, MALT and MALTER achieve comparable per-iteration wall-clock time and optimizer state memory to Muon. All time and memory measurements are conducted on the same single H100 GPU. The settings remain the same as previously stated, on a GPT-2 Medium (355M) model.

\section{Conclusion}
We proposed MALT, a lightweight curvature-aware Muon variant that performs
orthogonalization in a row-column diagonally preconditioned space.
By combining two-sided diagonal preconditioning, preconditioned
orthogonalization, and norm grafting, MALT addresses curvature anisotropy and
gradient anisotropy while preserving nearly the same memory and computational
cost as Muon. We further introduced MALTER, which incorporates a norm-based
noise-adaptive scalar stepsize in the preconditioned space. We established a
stochastic non-convex convergence guarantee for MALT and demonstrated on
GPT-2 pretraining that MALT and MALTER consistently outperform Muon, with
MALTER achieving the best overall validation performance. These results
indicate that curvature-aware preconditioning and noise-adaptive scale control
are useful and complementary ingredients for designing efficient matrix-aware
optimizers.

\bibliographystyle{plain}
\bibliography{cite}

@article{minaee2024large,
  title   = {Large Language Models: A Survey},
  author  = {Minaee, Shervin and Mikolov, Tomas and Nikzad, Narjes and Chenaghlu, Meysam and Socher, Richard and Amatriain, Xavier and Gao, Jianfeng},
  journal = {arXiv preprint arXiv:2402.06196},
  year    = {2024}
}

@article{zhao2023survey,
  title   = {A Survey of Large Language Models},
  author  = {Zhao, Wayne Xin and Zhou, Kun and Li, Junyi and Tang, Tianyi and Wang, Xiaolei and Hou, Yupeng and Min, Yingqian and Zhang, Beichen and Zhang, Junjie and Dong, Zican and others},
  journal = {arXiv preprint arXiv:2303.18223},
  year    = {2023}
}

@article{bommasani2021foundation,
  title   = {On the Opportunities and Risks of Foundation Models},
  author  = {Bommasani, Rishi and Hudson, Drew A. and Adeli, Ehsan and Altman, Russ and Arora, Simran and Arx, Sydney von and Bernstein, Michael S. and Bohg, Jeannette and Bosselut, Antoine and Brunskill, Emma and others},
  journal = {arXiv preprint arXiv:2108.07258},
  year    = {2021}
}

@misc{openai2026gpt55,
  title        = {{GPT-5.5 System Card}},
  author       = {{OpenAI}},
  year         = {2026},
  howpublished = {\url{https://openai.com/index/gpt-5-5-system-card/}},
  note         = {Accessed: 2026-06-07}
}

@article{deepseekai2025deepseekv32,
  title   = {{DeepSeek-V3.2}: Pushing the Frontier of Open Large Language Models},
  author  = {{DeepSeek-AI}},
  journal = {arXiv preprint arXiv:2512.02556},
  year    = {2025}
}

@book{nocedal2006numerical,
  title     = {Numerical Optimization},
  author    = {Nocedal, Jorge and Wright, Stephen J.},
  edition   = {2},
  publisher = {Springer},
  year      = {2006}
}

@inproceedings{ghorbani2019hessian,
  title     = {An Investigation into Neural Net Optimization via Hessian Eigenvalue Density},
  author    = {Ghorbani, Behrooz and Krishnan, Shankar and Xiao, Ying},
  booktitle = {Proceedings of the 36th International Conference on Machine Learning},
  pages     = {2232--2241},
  year      = {2019},
  volume    = {97},
  series    = {Proceedings of Machine Learning Research},
  publisher = {PMLR}
}

@inproceedings{sagun2018empirical,
  title     = {Empirical Analysis of the Hessian of Over-Parametrized Neural Networks},
  author    = {Sagun, Levent and Evci, Utku and G{\"u}ney, V. U{\u{g}}ur and Dauphin, Yann and Bottou, L{\'e}on},
  booktitle = {International Conference on Learning Representations Workshop},
  year      = {2018}
}

@article{ye2024preconditioning,
  title   = {Preconditioning for Accelerated Gradient Descent Optimization and Regularization},
  author  = {Ye, Qiang},
  journal = {arXiv preprint arXiv:2410.00232},
  year    = {2024}
}

@misc{metaai2025llama4,
  title        = {The {Llama 4} Herd: The Beginning of a New Era of Natively Multimodal AI Innovation},
  author       = {{Meta AI}},
  year         = {2025},
  howpublished = {\url{https://ai.meta.com/blog/llama-4-multimodal-intelligence/}},
  note         = {Accessed: 2026-06-07}
}

@article{geminiteam2025gemini25,
  title   = {{Gemini 2.5}: Pushing the Frontier with Advanced Reasoning, Multimodality, Long Context, and Next Generation Agentic Capabilities},
  author  = {{Gemini Team, Google}},
  journal = {arXiv preprint arXiv:2507.06261},
  year    = {2025}
}

@article{kaplan2020scaling,
  title   = {Scaling Laws for Neural Language Models},
  author  = {Kaplan, Jared and McCandlish, Sam and Henighan, Tom and Brown, Tom B. and Chess, Benjamin and Child, Rewon and Gray, Scott and Radford, Alec and Wu, Jeffrey and Amodei, Dario},
  journal = {arXiv preprint arXiv:2001.08361},
  year    = {2020}
}

@inproceedings{rajbhandari2020zero,
  title     = {{ZeRO}: Memory Optimizations Toward Training Trillion Parameter Models},
  author    = {Rajbhandari, Samyam and Rasley, Jeff and Ruwase, Olatunji and He, Yuxiong},
  booktitle = {Proceedings of the International Conference for High Performance Computing, Networking, Storage and Analysis},
  year      = {2020}
}

@article{narayanan2021efficient,
  title   = {Efficient Large-Scale Language Model Training on {GPU} Clusters Using {Megatron-LM}},
  author  = {Narayanan, Deepak and Shoeybi, Mohammad and Casper, Jared and LeGresley, Patrick and Patwary, Mostofa and Korthikanti, Vijay Anand and Vainbrand, Dmitri and Kashinkunti, Prethvi and Bernauer, Julie and Catanzaro, Bryan and Phanishayee, Amar and Zaharia, Matei},
  journal = {arXiv preprint arXiv:2104.04473},
  year    = {2021}
}

@article{patterson2021carbon,
  title   = {Carbon Emissions and Large Neural Network Training},
  author  = {Patterson, David and Gonzalez, Joseph and Le, Quoc and Liang, Chen and Munguia, Lluis-Miquel and Rothchild, Daniel and So, David and Texier, Maud and Dean, Jeff},
  journal = {arXiv preprint arXiv:2104.10350},
  year    = {2021}
}

@article{ranganath2026navigating,
  title   = {Navigating {LLM} Valley: From {AdamW} to Memory-Efficient and Matrix-Based Optimizers},
  author  = {Ranganath, Aditya},
  journal = {arXiv preprint arXiv:2605.09176},
  year    = {2026}
}

@article{jordan2024muon,
  title={Muon: An optimizer for hidden layers in neural networks, 2024},
  author={Jordan, Keller and Jin, Yuchen and Boza, Vlado and Jiacheng, You and Cesista, Franz and Newhouse, Laker and Bernstein, Jeremy},
  journal={URL https://kellerjordan. github. io/posts/muon},
  volume={6},
  number={3},
  pages={4},
  year={2024}
}

@book{nesterov2018lectures,
  title={Lectures on Convex Optimization},
  author={Nesterov, Yurii},
  year={2018},
  publisher={Springer}
}

@article{lau2025polargrad,
  title={Polargrad: A class of matrix-gradient optimizers from a unifying preconditioning perspective},
  author={Lau, Tim Tsz-Kit and Long, Qi and Su, Weijie},
  journal={arXiv preprint arXiv:2505.21799},
  year={2025}
}

@inproceedings{martens2015optimizing,
  title={Optimizing Neural Networks with Kronecker-factored Approximate Curvature},
  author={Martens, James and Grosse, Roger},
  booktitle={Proceedings of the 32nd International Conference on Machine Learning},
  pages={2408--2417},
  year={2015},
  volume={37},
  series={Proceedings of Machine Learning Research},
  publisher={PMLR}
}

@article{duchi2011adaptive,
  title={Adaptive Subgradient Methods for Online Learning and Stochastic Optimization},
  author={Duchi, John and Hazan, Elad and Singer, Yoram},
  journal={Journal of Machine Learning Research},
  volume={12},
  number={61},
  pages={2121--2159},
  year={2011}
}

@inproceedings{kingma2015adam,
  title={Adam: A Method for Stochastic Optimization},
  author={Kingma, Diederik P. and Ba, Jimmy},
  booktitle={International Conference on Learning Representations},
  year={2015}
}

@misc{tieleman2012rmsprop,
  title={Lecture 6.5---RMSProp: Divide the Gradient by a Running Average of Its Recent Magnitude},
  author={Tieleman, Tijmen and Hinton, Geoffrey},
  howpublished={COURSERA: Neural Networks for Machine Learning},
  year={2012}
}

@inproceedings{shazeer2018adafactor,
  title={Adafactor: Adaptive learning rates with sublinear memory cost},
  author={Shazeer, Noam and Stern, Mitchell},
  booktitle={International conference on machine learning},
  pages={4596--4604},
  year={2018},
  organization={PMLR}
}

@inproceedings{anil2019memory,
  title={Memory Efficient Adaptive Optimization},
  author={Anil, Rohan and Gupta, Vineet and Koren, Tomer and Singer, Yoram},
  booktitle={Advances in Neural Information Processing Systems},
  year={2019}
}

@inproceedings{zhang2025adammini,
  title={Adam-mini: Use Fewer Learning Rates To Gain More},
  author={Zhang, Yushun and Chen, Congliang and Li, Ziniu and Ding, Tian and Wu, Chenwei and Kingma, Diederik P. and Ye, Yinyu and Luo, Zhi-Quan and Sun, Ruoyu},
  booktitle={International Conference on Learning Representations},
  year={2025}
}

@article{liu2025muon,
  title={Muon is scalable for llm training},
  author={Liu, Jingyuan and Su, Jianlin and Yao, Xingcheng and Jiang, Zhejun and Lai, Guokun and Du, Yulun and Qin, Yidao and Xu, Weixin and Lu, Enzhe and Yan, Junjie and others},
  journal={arXiv preprint arXiv:2502.16982},
  year={2025}
}

@misc{zhang2026gram,
  title={Gram Newton--Schulz: A Fast, Hardware-Aware Newton--Schulz Algorithm for Muon},
  author={Jack Zhang and Noah Amsel and Berlin Chen and Tri Dao},
  howpublished={\url{https://tridao.me/blog/2026/gram-newton-schulz/}},
  year={2026}
}

@article{amsel2025polar,
  title   = {The Polar Express: Optimal Matrix Sign Methods and Their Application to the Muon Algorithm},
  author  = {Amsel, Noah and Persson, David and Musco, Christopher and Gower, Robert M.},
  journal = {arXiv preprint arXiv:2505.16932},
  year    = {2025}
}

@article{grishina2025accelerating,
  title   = {Accelerating Newton-Schulz Iteration for Orthogonalization via Chebyshev-type Polynomials},
  author  = {Grishina, Ekaterina and Smirnov, Matvey and Rakhuba, Maxim},
  journal = {arXiv preprint arXiv:2506.10935},
  year    = {2025}
}

@inproceedings{dauphin2014identifying,
  title     = {Identifying and Attacking the Saddle Point Problem in High-Dimensional Non-Convex Optimization},
  author    = {Dauphin, Yann N. and Pascanu, Razvan and Gulcehre, Caglar and Cho, Kyunghyun and Ganguli, Surya and Bengio, Yoshua},
  booktitle = {Advances in Neural Information Processing Systems},
  volume    = {27},
  year      = {2014}
}

@article{refael2026sumo,
  title={Sumo: Subspace-aware moment-orthogonalization for accelerating memory-efficient llm training},
  author={Refael, Yehonathan and Smorodinsky, Guy and Tirer, Tom and Lindenbaum, Ofir},
  journal={Advances in Neural Information Processing Systems},
  volume={38},
  pages={147250--147281},
  year={2026}
}

@article{he2025low,
  title={Low-rank orthogonalization for large-scale matrix optimization with applications to foundation model training},
  author={He, Chuan and Deng, Zhanwang and Lu, Zhaosong},
  journal={arXiv preprint arXiv:2509.11983},
  year={2025}
}

@article{nguyen2026spectral,
  title={Spectral Flattening Is All Muon Needs: How Orthogonalization Controls Learning Rate and Convergence},
  author={Nguyen, Tien-Phat and Nguyen, Truong and Truong, Minh-Phuc and Nguyen, Tuc and Bailey, James and Le, Trung},
  journal={arXiv preprint arXiv:2605.13079},
  year={2026}
}

@article{si2025adamuon,
  title={Adamuon: Adaptive muon optimizer},
  author={Si, Chongjie and Zhang, Debing and Shen, Wei},
  journal={arXiv preprint arXiv:2507.11005},
  year={2025}
}

@article{zhang2026namo,
  title={Adam improves muon: Adaptive moment estimation with orthogonalized momentum},
  author={Zhang, Minxin and Liu, Yuxuan and Schaeffer, Hayden},
  journal={arXiv preprint arXiv:2602.17080},
  year={2026}
}

@article{li2026variance,
  title={Variance-Adaptive Muon: Accelerating LLM Pretraining with NSR-Modulated and Variance-Scaled Momentum},
  author={Li, Jingru and Fan, Yibo and Li, Huan},
  journal={arXiv preprint arXiv:2601.14603},
  year={2026}
}

@article{song2026decoupling,
  title={Decoupling Variance and Scale-Invariant Updates in Adaptive Gradient Descent for Unified Vector and Matrix Optimization},
  author={Song, Zitao and Bai, Cedar Site and Zhang, Zhe and Bullins, Brian and Gleich, David F},
  journal={arXiv preprint arXiv:2602.06880},
  year={2026}
}

@article{cheng2026trasmuon,
  title={TrasMuon: Trust-Region Adaptive Scaling for Orthogonalized Momentum Optimizers},
  author={Cheng, Peng and Zang, Jiucheng and Li, Qingnan and Ma, Liheng and Cui, Yufei and Zhang, Yingxue and Chen, Boxing and Jian, Ming and Tong, Wen},
  journal={arXiv preprint arXiv:2602.13498},
  year={2026}
}

@article{li2025normuon,
  title={NorMuon: Making Muon more efficient and scalable},
  author={Li, Zichong and Liu, Liming and Liang, Chen and Chen, Weizhu and Zhao, Tuo},
  journal={arXiv preprint arXiv:2510.05491},
  year={2025}
}

@article{liu2026muon2,
  title={Muon $^{2}$: Boosting Muon via Adaptive Second-Moment Preconditioning},
  author={Liu, Ziyue and Zhang, Ruijie and Wang, Zhengyang and Zhao, Yequan and Su, Yupeng and Yang, Zi and Zhang, Zheng},
  journal={arXiv preprint arXiv:2604.09967},
  year={2026}
}

@article{aurora2026,
  title={Aurora: A Leverage-Aware Optimizer for Rectangular Matrices, 2026},
  author={Dewulf, Alec and Pai, Dhruv and Yang, Li and Zhang, Ashley and Keigwin, Ben},
  journal={URL https://blog.tilderesearch.com/blog/aurora},
  year={2026}
}

@article{deng2026rmnp,
  title={Rmnp: Row-momentum normalized preconditioning for scalable matrix-based optimization},
  author={Deng, Shenyang and Ouyang, Zhuoli and Pang, Tianyu and Liu, Zihang and Jin, Ruochen and Yu, Shuhua and Yang, Yaoqing},
  journal={arXiv preprint arXiv:2603.20527},
  year={2026}
}

@article{lion2026muown,
  title={Muown: Row-Norm Control for Muon Optimization},
  author={Lion, Kai and H{\"u}bler, Florian and Li, Bingcong and Orvieto, Antonio and He, Niao},
  journal={arXiv preprint arXiv:2605.10797},
  year={2026}
}

@article{yuan2026nora,
  title={Nora: Normalized Orthogonal Row Alignment for Scalable Matrix Optimizer},
  author={Yuan, Jinghui and Zou, Jiaxuan and Wang, Shuo and Liu, Yong and Nie, Feiping},
  journal={arXiv preprint arXiv:2605.03769},
  year={2026}
}

@article{zhang2026muonplus,
  title={Muon+: Towards better muon via one additional normalization step},
  author={Zhang, Ruijie and Zhao, Yequan and Liu, Ziyue and Wang, Zhengyang and Zhang, Zheng},
  journal={arXiv preprint arXiv:2602.21545},
  year={2026}
}

@article{chang2026muoneq,
  title={Muoneq: Balancing before orthogonalization with lightweight equilibration},
  author={Chang, Da and Shi, Qiankun and Zhang, Lvgang and Li, Yu and Zhang, Ruijie and Lu, Yao and Liu, Yongxiang and Yuan, Ganzhao},
  journal={arXiv preprint arXiv:2603.28254},
  year={2026}
}

@inproceedings{yao2021adahessian,
  title={AdaHessian: An Adaptive Second Order Optimizer for Machine Learning},
  author={Yao, Zhewei and Gholami, Amir and Shen, Sheng and Mustafa, Mustafa and Keutzer, Kurt and Mahoney, Michael W.},
  booktitle={Proceedings of the AAAI Conference on Artificial Intelligence},
  volume={35},
  number={12},
  pages={10665--10673},
  year={2021}
}

@inproceedings{liu2024sophia,
  title={Sophia: A Scalable Stochastic Second-order Optimizer for Language Model Pre-training},
  author={Liu, Hong and Li, Zhiyuan and Hall, David and Liang, Percy and Ma, Tengyu},
  booktitle={International Conference on Learning Representations},
  year={2024}
}

@article{huang2026spectra,
  title={Spectra: Rethinking Optimizers for LLMs Under Spectral Anisotropy},
  author={Huang, Zhendong and Cao, Hengjie and Dong, Fang and Huang, Ruijun and Chen, Mengyi and Yang, Yifeng and Zhang, Xin and Chen, Anrui and Dong, Mingzhi and Wang, Yujiang and others},
  journal={arXiv preprint arXiv:2602.11185},
  year={2026}
}

@inproceedings{gupta2018shampoo,
  title={Shampoo: Preconditioned Stochastic Tensor Optimization},
  author={Gupta, Vineet and Koren, Tomer and Singer, Yoram},
  booktitle={Proceedings of the 35th International Conference on Machine Learning},
  pages={1842--1850},
  year={2018},
  volume={80},
  series={Proceedings of Machine Learning Research},
  publisher={PMLR}
}

@inproceedings{duvvuri2024caspr,
  title={Combining Axes Preconditioners through Kronecker Approximation for Deep Learning},
  author={Duvvuri, Sai Surya and Devvrit, Fnu and Anil, Rohan and Hsieh, Cho-Jui and Dhillon, Inderjit S.},
  booktitle={International Conference on Learning Representations},
  year={2024}
}

@inproceedings{vyas2025soap,
  title={SOAP: Improving and Stabilizing Shampoo using Adam for Language Modeling},
  author={Vyas, Nikhil and Morwani, Depen and Zhao, Rosie and Kwun, Mujin and Shapira, Itai and Brandfonbrener, David and Janson, Lucas and Kakade, Sham M.},
  booktitle={International Conference on Learning Representations},
  year={2025}
}

@article{xu2026fismo,
  title={FISMO: Fisher-Structured Momentum-Orthogonalized Optimizer},
  author={Xu, Chenrui and Yan, Wenjing and Zhang, Ying-Jun Angela},
  journal={arXiv preprint arXiv:2601.21750},
  year={2026}
}

@article{zhang2026mousse,
  title={Mousse: Rectifying the geometry of muon with curvature-aware preconditioning},
  author={Zhang, Yechen and Xing, Shuhao and Huang, Junhao and Lv, Kai and Zhou, Yunhua and Qiu, Xipeng and Guo, Qipeng and Chen, Kai},
  journal={arXiv preprint arXiv:2603.09697},
  year={2026}
}

@article{du2026newton,
  title={The newton-muon optimizer},
  author={Du, Zhehang and Su, Weijie},
  journal={arXiv preprint arXiv:2604.01472},
  year={2026}
  }

@article{amari1998natural,
  title={Natural Gradient Works Efficiently in Learning},
  author={Amari, Shun-ichi},
  journal={Neural Computation},
  volume={10},
  number={2},
  pages={251--276},
  year={1998}
}

@article{schraudolph2002fast,
  title={Fast Curvature Matrix-Vector Products for Second-Order Gradient Descent},
  author={Schraudolph, Nicol N.},
  journal={Neural Computation},
  volume={14},
  number={7},
  pages={1723--1738},
  year={2002}
}

@inproceedings{martens2010deep,
  title={Deep Learning via Hessian-free Optimization},
  author={Martens, James},
  booktitle={Proceedings of the 27th International Conference on Machine Learning},
  pages={735--742},
  year={2010}
}

@inproceedings{loshchilov2019decoupled,
  title={Decoupled Weight Decay Regularization},
  author={Loshchilov, Ilya and Hutter, Frank},
  booktitle={International Conference on Learning Representations},
  year={2019}
}

@inproceedings{botev2017practical,
  title={Practical Gauss-Newton optimisation for deep learning},
  author={Botev, Aleksandar and Ritter, Hippolyt and Barber, David},
  booktitle={International Conference on Machine Learning},
  pages={557--565},
  year={2017},
  organization={PMLR}
}

@article{zhang2024transformers,
  title={Why transformers need adam: A hessian perspective},
  author={Zhang, Yushun and Chen, Congliang and Ding, Tian and Li, Ziniu and Sun, Ruoyu and Luo, Zhi-Quan},
  journal={Advances in neural information processing systems},
  volume={37},
  pages={131786--131823},
  year={2024}
}

@article{dong2025towards,
  title={Towards quantifying the hessian structure of neural networks},
  author={Dong, Zhaorui and Zhang, Yushun and Yao, Jianfeng and Sun, Ruoyu},
  journal={arXiv preprint arXiv:2505.02809},
  year={2025}
}

@misc{karpathy2022nanogpt,
  author       = {Karpathy, Andrej},
  title        = {nanoGPT},
  year         = {2022},
  howpublished = {\url{https://github.com/karpathy/nanoGPT}},
  note         = {GitHub repository}
}

@misc{gokaslan2019openwebtext,
  title={Openwebtext corpus},
  author={Gokaslan, Aaron and Cohen, Vanya and Pavlick, Ellie and Tellex, Stefanie},
  year={2019}
}

@inproceedings{bernstein2018signsgd,
  title={signSGD: Compressed Optimisation for Non-Convex Problems},
  author={Bernstein, Jeremy and Wang, Yu-Xiang and Azizzadenesheli, Kamyar and Anandkumar, Anima},
  booktitle={International Conference on Machine Learning},
  pages={560--569},
  year={2018},
  organization={PMLR}
}

@article{agarwal2020disentangling,
  title={Disentangling adaptive gradient methods from learning rates},
  author={Agarwal, Naman and Anil, Rohan and Hazan, Elad and Koren, Tomer and Zhang, Cyril},
  journal={arXiv preprint arXiv:2002.11803},
  year={2020}
}

@article{anil2020scalable,
  title={Scalable second order optimization for deep learning},
  author={Anil, Rohan and Gupta, Vineet and Koren, Tomer and Regan, Kevin and Singer, Yoram},
  journal={arXiv preprint arXiv:2002.09018},
  year={2020}
}

@article{frans2025really,
  title={What really matters in matrix-whitening optimizers?},
  author={Frans, Kevin and Abbeel, Pieter and Levine, Sergey},
  journal={arXiv preprint arXiv:2510.25000},
  year={2025}
}

@article{frans2026stable,
  title={A stable whitening optimizer for efficient neural network training},
  author={Frans, Kevin and Levine, Sergey and Abbeel, Pieter},
  journal={Advances in Neural Information Processing Systems},
  volume={38},
  pages={174086--174110},
  year={2026}
}

@article{shen2025convergence,
  title={On the convergence analysis of muon},
  author={Shen, Wei and Huang, Ruichuan and Huang, Minhui and Shen, Cong and Zhang, Jiawei},
  journal={arXiv preprint arXiv:2505.23737},
  year={2025}
}

@article{sagun2016eigenvalues,
  title={Eigenvalues of the Hessian in Deep Learning: Singularity and Beyond},
  author={Sagun, Levent and Bottou, Leon and LeCun, Yann},
  journal={arXiv preprint arXiv:1611.07476},
  year={2016}
}

@article{sagun2017empirical,
  title={Empirical Analysis of the Hessian of Over-Parametrized Neural Networks},
  author={Sagun, Levent and Evci, Utku and Guney, V. Ugur and Dauphin, Yann and Bottou, Leon},
  journal={arXiv preprint arXiv:1706.04454},
  year={2017}
}

@inproceedings{tang2025overlooked,
  title={Investigating the Overlooked Hessian Structure},
  author={Tang, Q. Y. and others},
  booktitle={International Conference on Machine Learning},
  year={2025}
}

@inproceedings{ben2024high,
  title={High-dimensional SGD aligns with emerging outlier eigenspaces},
  author={Ben Arous, Gerard and Gheissari, Reza and Huang, Jiaoyang and Jagannath, Aukosh},
  booktitle={International Conference on Learning Representations},
  volume={2024},
  pages={47732--47778},
  year={2024}
}

@inproceedings{ibrahim2024simple,
  title={Simple and Scalable Strategies to Continually Pre-train Large Language Models},
  author={Ibrahim, Adam and Th{\'e}rien, Benjamin and Gupta, Kshitij and Richter, Mats L. and Anthony, Quentin and Lesort, Timoth{\'e}e and Belilovsky, Eugene and Rish, Irina},
  booktitle={International Conference on Learning Representations},
  year={2024}
}

@inproceedings{wen2024understanding,
  title={Understanding Warmup-Stable-Decay Learning Rates: A River Valley Loss Landscape Perspective},
  author={Wen, Kaiyue and Li, Zhiyuan and Wang, Jason and Hall, David and Liang, Percy and Ma, Tengyu},
  booktitle={International Conference on Learning Representations},
  year={2025}
}

@article{nagashima2026improved,
  title={Improved Convergence Rates of Muon Optimizer for Nonconvex Optimization},
  author={Nagashima, Shuntaro and Iiduka, Hideaki},
  journal={arXiv preprint arXiv:2601.19400},
  year={2026}
}

\newpage
\appendix

\section{Proof of Theorem \ref{thm-1}}
Based on Assumption \ref{assum-2}, there is
\begin{align}\label{descent-1}
\mathbb E \left[ f\left(  \boldsymbol W_{t}\right) \right] & \leq \mathbb E \left[ f\left(  \boldsymbol W_{t-1}\right) \right] + \mathbb E \left[\langle \nabla f\left(  \boldsymbol W_{t-1}\right)  , \boldsymbol W_{t} - \boldsymbol W_{t-1}  \rangle \right] + \frac{ L}{2} \mathbb E \left[\left\|  \boldsymbol W_{t} - \boldsymbol W_{t-1}  \right\|_F^2 \right] \nonumber \\
& \leq  \mathbb E \left[ f\left(  \boldsymbol W_{t-1}\right) \right] - \mathbb E \left[\frac{ \sqrt n \eta }{ \left\| \boldsymbol D_t\right\|_F + \epsilon} \langle  \nabla f\left(  \boldsymbol W_{t-1}\right), \boldsymbol L_t \boldsymbol O_t \boldsymbol R_t  \rangle \right] + \frac L 2 n\eta^2 \nonumber \\
& = \mathbb E \left[ f\left(  \boldsymbol W_{t-1}\right) \right]  - \mathbb E \left[\frac{ \sqrt n \eta }{ \left\| \boldsymbol L_t \boldsymbol O_t \boldsymbol R_t\right\|_F + \epsilon}  \langle \boldsymbol L_t \left( \nabla f \left( \boldsymbol W_{t-1} \right) - \boldsymbol M_t \right) \boldsymbol R_t , \boldsymbol O_t\rangle \right] \nonumber \\
& \quad - \mathbb E \left[ \frac{ \sqrt n \eta }{ \left\| \boldsymbol L_t \boldsymbol O_t \boldsymbol R_t \right\|_F + \epsilon}  \langle \boldsymbol L_t \boldsymbol M_t \boldsymbol R_t  ,  \operatorname{Orth} \left( \boldsymbol L_t \boldsymbol M_t \boldsymbol R_t  \right) \rangle   + \frac L 2 n\eta^2 \right] \nonumber \\
& \overset{(i)}{\leq} \mathbb E \left[ f\left(  \boldsymbol W_{t-1}\right) \right] + \mathbb E \left[\frac{ n \eta} { \left\| \boldsymbol L_t \boldsymbol O_t \boldsymbol R_t  \right\|_F +\epsilon } \left\|   \boldsymbol L_t \left( \nabla f \left( \boldsymbol W_{t-1} \right) - \boldsymbol M_t \right) \boldsymbol R_t \right\|_F \right]  +\frac L 2 n\eta^2 \nonumber \\
& \quad  - \mathbb E \left[ \frac{ \sqrt n \eta} { \left\| \boldsymbol L_t \boldsymbol O_t \boldsymbol R_t  \right\|_F +\epsilon } \left\|  \boldsymbol L_t \boldsymbol M_t \boldsymbol R_t   \right\|_* \right] \nonumber \\
& \overset{(ii)}{\leq}  \mathbb E \left[ f\left(  \boldsymbol W_{t-1}\right) \right]  + \mathbb E \left[ \frac{ n \eta} { \left\| \boldsymbol L_t \boldsymbol O_t \boldsymbol R_t  \right\|_F +\epsilon } \left\|   \boldsymbol L_t \left( \nabla f \left( \boldsymbol W_{t-1} \right) - \boldsymbol M_t \right) \boldsymbol R_t \right\|_F \right]  +  \frac L 2 n\eta^2 \nonumber \\
& \quad - \mathbb E \left[ \frac{ \sqrt n \eta} { \left\| \boldsymbol L_t \boldsymbol O_t \boldsymbol R_t  \right\|_F +\epsilon }  \left(   \left\| \boldsymbol L_t \nabla f \left( \boldsymbol W_{t-1} \right) \boldsymbol R_t \right\|_* -  \left\| \boldsymbol L_t \left( \nabla f \left( \boldsymbol W_{t-1} \right) - \boldsymbol M_t \right) \boldsymbol R_t \right\|_*  \right) \right]\nonumber \\
& \overset{(iii)}{\leq}   \mathbb E \left[ f\left(  \boldsymbol W_{t-1}\right) \right] + 2  n \epsilon^{-1} \eta  \mathbb E \left[ \left\|   \boldsymbol L_t \left( \nabla f \left( \boldsymbol W_{t-1} \right) - \boldsymbol M_t \right) \boldsymbol R_t \right\|_F \right]   + \frac L 2 n \eta^2\nonumber \\
& \quad - \mathbb E \left[ \frac{ \sqrt n \eta} { \left\| \boldsymbol L_t \boldsymbol O_t \boldsymbol R_t  \right\|_F +\epsilon } \left\| \boldsymbol L_t\nabla f \left(  \boldsymbol W_{t-1}\right)  \boldsymbol R_t  \right\|_F \right] \nonumber \\
 & \overset{(iv)}{\leq}   \mathbb E \left[ f\left(  \boldsymbol W_{t-1}\right) \right] + 2 n \epsilon^{-5/4} \eta  \mathbb E \left[ \left\|   \nabla f \left( \boldsymbol W_{t-1} \right) - \boldsymbol M_t  \right\|_F \right]     + \frac L 2 n \eta^2 \nonumber \\
 & \quad - \frac{ \sqrt n \eta} {  \epsilon^{-1/4} \sqrt n  +\epsilon } \mathbb E \left[ \left\| \boldsymbol L_t\nabla f \left(  \boldsymbol W_{t-1}\right)  \boldsymbol R_t  \right\|_F \right] \nonumber \\
 & \leq   \mathbb E \left[ f\left(  \boldsymbol W_{t-1}\right) \right] + 2 n \epsilon^{-5/4} \eta \mathbb E \left[ \left\|   \nabla f \left( \boldsymbol W_{t-1} \right) - \boldsymbol M_t  \right\|_F \right]   - \frac{\epsilon^{1/4}}{2} \eta  \mathbb E \left[ \left\| \boldsymbol L_t\nabla f \left(  \boldsymbol W_{t-1}\right)  \boldsymbol R_t  \right\|_F \right]  + \frac L 2 n \eta^2, 
\end{align}
where (i) and (ii) are due to Cauchy-Schwarz and triangle inequalities, respectively. (iii) is because norm inequality $\left\| \boldsymbol X \right\|_F\leq \left\| \boldsymbol X \right\|_*\leq \sqrt{ \operatorname{rank}\left( \boldsymbol X\right) } \left\| \boldsymbol X \right\|_F $. (iv) and last inequality are  due to $\left\| \boldsymbol L_t \right\| \leq \epsilon^{-1/8}, \left\| \boldsymbol R_t \right\| \leq \epsilon^{-1/8}, \epsilon \in (0,1)$. \\

Now we come to lower bound the term $ \mathbb E \left[\left\| \boldsymbol L_t \nabla f \left( \boldsymbol W_{t-1} \right) \boldsymbol R_t \right\|_F \right]$ as 
\begin{align}\label{lower-bound}
& \mathbb E \left[\left\| \boldsymbol L_t \nabla f \left( \boldsymbol W_{t-1} \right) \boldsymbol R_t \right\|_{F} \right]  \geq \mathbb E \left[ \sigma_{\min} \left(  \boldsymbol L_t\right) \sigma_{\min} \left(  \boldsymbol R_t\right) \left\|  \nabla f \left( \boldsymbol W_{t-1} \right)  \right\|_F  \right] \nonumber \\
& \qquad \qquad \qquad \qquad \qquad\geq  \mathbb E \left [  \frac{1}{\left( \left\| \boldsymbol l_t\right\|_1 +\epsilon \right)^{\frac 1 8}} \frac{1}{\left( \left\| \boldsymbol r_t\right\|_1 +\epsilon \right)^{\frac 1 8}}  \left\| \nabla f \left( \boldsymbol W_{t-1} \right) \right\|_{F}  \right] \nonumber \\
&\qquad \qquad \qquad \qquad \qquad \overset{(a)}{\geq} \frac{ \mathbb E^2 \left[  \left\| \nabla f \left( \boldsymbol W_{t-1} \right) \right\|_{F}^{\frac 1 2}  \right]   }{ \mathbb E \left[  \left( \left\| \boldsymbol l_t\right\|_1 + \epsilon \right)^{\frac 1 4}  \right]  } \nonumber \\
& \qquad \qquad \qquad \qquad \qquad\overset{(b)}{\geq}   \frac{ \mathbb E^2 \left[  \left\| \nabla f \left( \boldsymbol W_{t-1} \right) \right\|_{F}^{\frac 1 2}  \right]   }{   \mathbb E \left[ \left\| \boldsymbol l_t \right\|_1^{\frac 1 4} \right] + \epsilon^{\frac 1 4}   } \nonumber \\
& \qquad \qquad \qquad \qquad \qquad =  \frac{ \mathbb E^2 \left[  \left\| \nabla f \left( \boldsymbol W_{t-1} \right) \right\|_{F}^{\frac 1 2}  \right]   } {   \left( 1 - \mu_2^t \right)^{\frac 1 4} \mathbb E \left[ \sqrt[4]{  \sum_{\tau=1}^t \frac{1-\mu_2}{1-\mu_2^t} \mu_2^{t-\tau} \left\| \boldsymbol G_\tau \right\|_F^2      } \right]      + \epsilon^{\frac 1 4} } \nonumber \\
& \quad \overset{(c)}{\geq} \frac{ \mathbb E^2 \left[  \left\| \nabla f \left( \boldsymbol W_{t-1} \right) \right\|_{F}^{\frac 1 2}  \right] }{  \left( 1 - \mu_2^t\right)^{\frac 1 4} \mathbb E \left[ \sqrt{ \sqrt{  \sum_{\tau=1}^t \frac{1-\mu_2}{1-\mu_2^t} \mu_2^{t-\tau} \left\| \boldsymbol G_\tau - \nabla f\left( \boldsymbol W_{\tau-1}\right) \right\|_F^2  } + \sqrt{ \sum_{\tau=1}^t \frac{1-\mu_2}{1-\mu_2^t} \mu_2^{t-\tau} \left\|\nabla f\left( \boldsymbol W_{\tau -1}\right) \right\|_F^2  }} \right]  + \epsilon^{\frac 1 4}} \nonumber \\
& \quad \overset{(d)}{\geq} \frac{ \mathbb E^2 \left[  \left\| \nabla f \left( \boldsymbol W_{t-1} \right) \right\|_{F}^{\frac 1 2}  \right] } { \left( 1 - \mu_2^t\right)^{\frac 1 4} \left( \mathbb E \left[ \sqrt[4]{   \sum_{\tau=1}^t \frac{1-\mu_2}{1-\mu_2^t} \mu_2^{t-\tau} \left\| \boldsymbol G_\tau - \nabla f\left( \boldsymbol W_{\tau-1}\right) \right\|_F^2 } \right] + \mathbb E\left[ \sqrt[4]{  \sum_{\tau=1}^t \frac{1-\mu_2}{1-\mu_2^t} \mu_2^{t-\tau} \left\|\nabla f\left( \boldsymbol W_{\tau -1}\right) \right\|_F^2   } \right]\right)  + \epsilon^{\frac 1 4}    }\nonumber \\
& \quad \overset{(e)}{\geq} \frac{ \mathbb E^2 \left[  \left\| \nabla f \left( \boldsymbol W_{t-1} \right) \right\|_{F}^{\frac 1 2}  \right] } {  \left( 1 - \mu_2^t\right)^{\frac 1 4} \left( \mathbb E^{\frac 1 4} \left[{   \sum_{\tau=1}^t \frac{1-\mu_2}{1-\mu_2^t} \mu_2^{t-\tau} \left\| \boldsymbol G_\tau - \nabla f\left( \boldsymbol W_{\tau-1}\right) \right\|_F^2 } \right] + \mathbb E^{\frac 1 4}\left[ {  \sum_{\tau=1}^t \frac{1-\mu_2}{1-\mu_2^t} \mu_2^{t-\tau} \left\|\nabla f\left( \boldsymbol W_{\tau -1}\right) \right\|_F^2   } \right]\right)  + \epsilon^{\frac 1 4}    } \nonumber \\
& \quad \overset{(f)}{\geq}\frac{ \left( 1 - \mu_2^t\right)^{ - \frac 1 4}\mathbb E^2 \left[  \left\| \nabla f \left( \boldsymbol W_{t-1} \right) \right\|_{F}^{\frac 1 2}  \right] }  {   \frac{\sqrt{\sigma}}{B^{\frac 1 4}} + \mathbb E \left[ \sqrt{ \sqrt{\sum\limits_{\tau=1}^{t} \frac{1-\mu_2}{1-\mu_2^t} \mu_2^{t-\tau} \left\|\nabla f\left( \boldsymbol W_{t -1}\right) \right\|_F^2} + \sqrt{ \sum\limits_{\tau=1}^{t} \frac{1-\mu_2}{1-\mu_2^t} \mu_2^{t-\tau} \left\|\nabla f\left( \boldsymbol W_{\tau -1}\right) -\nabla f\left( \boldsymbol W_{t -1}\right) \right\|_F^2    } } \right]   + \frac{\epsilon^{\frac 1 4}} {\left(1-\mu_2^t\right)^{\frac 1 4}}   } \nonumber \\
& \quad \overset{(g)}{\geq} \frac{ \left( 1 - \mu_2^t\right)^{ - \frac 1 4}\mathbb E^2 \left[  \left\| \nabla f \left( \boldsymbol W_{t-1} \right) \right\|_{F}^{\frac 1 2}  \right] } {  \frac{\sqrt{\sigma}}{B^{\frac 1 4}} + \mathbb E \left[ \left\| \nabla f\left( \boldsymbol W_{t-1} \right) \right\|_F^{\frac 1 2}\right] + \mathbb E^{\frac 1 4} \left[ \sum_{\tau=1}^t \frac{1-\mu_2}{1-\mu_2^t}\mu_2^{t-\tau}\left( t-\tau\right)^2L^2\eta^2 n  \right]  + \frac{\epsilon^{\frac 1 4}} {\left(1-\mu_2^t\right)^{\frac 1 4}}   } \nonumber \\
& \quad \overset{(h)}{\geq} \frac{ \mathbb E^2 \left[  \left\| \nabla f \left( \boldsymbol W_{t-1} \right) \right\|_{F}^{\frac 1 2}  \right] } {  \frac{\sqrt{\sigma}}{B^{\frac 1 4}} + \mathbb E \left[ \left\| \nabla f\left( \boldsymbol W_{t-1} \right) \right\|_F^{\frac 1 2}\right] +  \left( 1 - \mu_2^t\right)^{\frac 1 4} \cdot \sqrt{\eta L} n^{\frac 1 4} \left( \frac{1-\mu_2}{1-\mu_2^t} \right)^{\frac 1 4} \left( \frac{2}{\left(1-\mu_2\right)^3} \right)^{\frac 1 4}  + \epsilon^{\frac 1 4}   } \nonumber \\
& \quad {\geq} \frac{ \mathbb E^2 \left[  \left\| \nabla f \left( \boldsymbol W_{t-1} \right) \right\|_{F}^{\frac 1 2}  \right]  }{  \mathbb E \left[ \left\| \nabla f\left( \boldsymbol W_{t-1} \right) \right\|_F^{\frac 1 2}\right] + 1 + 2n^{\frac 1 4} \sqrt{\frac{\eta L}{1-\mu_2}} +\frac{\sqrt{\sigma}}{B^{\frac 1 4}}      },\nonumber \\
& \quad \; \geq \frac{1}{2} \min \left\{ \mathbb E \left[ \left\| \nabla f\left( \boldsymbol W_{t-1} \right) \right\|_F^{\frac 1 2}\right] , \frac{ \mathbb E^2 \left[ \left\| \nabla f\left( \boldsymbol W_{t-1} \right) \right\|_F^{\frac 1 2}\right]}{  1 + 2n^{\frac 1 4} \sqrt{\frac{\eta L}{1-\mu_2}} + \frac{\sqrt{\sigma}}{B^{\frac 1 4}}  }\right\},
\end{align}
where (a) is because $\left\| \boldsymbol l_t \right\|_1 = \left\| \boldsymbol r_t \right\|_1$ and Cauchy–Schwarz inequality $\mathbb E ^2 \left[ XY\right] \leq \mathbb E \left[ X^2 \right]\mathbb E \left[ Y^2 \right]$. (b) is due to $ \left( a+b\right)^p \leq a^p + b^p, \forall a,b\geq 0, p \in (0,1)$. (c) is based on $\sum_{\tau=1}^t \frac{1-\mu_2}{1-\mu_2^t}\mu_2^{t-\tau}$ and Minkowski inequality. (d) is due to $\sqrt{a+b} \leq \sqrt a + \sqrt b$. (e) is because Jensen inequality $\mathbb E \left[ X^{\frac 1 4}\right]\leq \mathbb E^{\frac 1 4}\left[ X\right],\forall X \geq 0$. (f) uses the Assumption \ref{assum-3} and Minkowski inequality again. (g) uses the Assumption \ref{assum-2}, 11th line of Algorithm \ref{algo-1}, $\sum_{\tau=1}^t \frac{1-\mu_2}{1-\mu_2^t}\mu_2^{t-\tau}$ and $\sqrt{a+b} \leq \sqrt a + \sqrt b$ again. (h) is due to $ \sum_{\tau=1}^\infty \mu_2^\tau \tau^2 \leq \frac{2}{\left(1 - \mu_2 \right)^3}$. The last two inequalities are due to $\epsilon \in (0,1)$ and inequality $  \frac{x^2}{ax+b}
\ge
\frac{x^2}{2\max\{ax,b\}}
=
\frac12\min\left\{\frac{x}{a},\frac{x^2}{b}\right\}, \forall a,b>0, \forall x \geq 0$.   

 Substituting \eqref{lower-bound} into \eqref{descent-1} has
\begin{align}\label{final-descent}
\mathbb E \left[ f\left( \boldsymbol W_t  \right) \right] & \leq \mathbb E \left[ f\left(  \boldsymbol W_{t-1}\right) \right] + 2 n \epsilon^{-5/4} \eta \mathbb E \left[ \left\|   \nabla f \left( \boldsymbol W_{t-1} \right) - \boldsymbol M_t  \right\|_F \right] + \frac L 2 n \eta^2 \nonumber \\
& \quad - \frac{ \epsilon^{1/4} \eta }{4} \min \left\{ \mathbb E \left[ \left\| \nabla f\left( \boldsymbol W_{t-1} \right) \right\|_F^{\frac 1 2}\right] , \frac{ \mathbb E^2 \left[ \left\| \nabla f\left( \boldsymbol W_{t-1} \right) \right\|_F^{\frac 1 2}\right]}{  1 + 2n^{\frac 1 4} \sqrt{\frac{\eta L}{1-\mu_2}} + \frac{\sqrt{\sigma}}{B^{\frac 1 4}}  }\right\}.
\end{align}

Further substituting the result in Lemma \ref{lemma-2} into \eqref{final-descent} and summing over \(t=1,\ldots,T\) yield
\begin{align}
\frac{\epsilon^{1/4}\eta}{4} &\sum_{t=1}^{T}  \min \left\{ \mathbb E \left[ \left\| \nabla f\left( \boldsymbol W_{t-1} \right) \right\|_F^{\frac 1 2}\right] , \frac{ \mathbb E^2 \left[ \left\| \nabla f\left( \boldsymbol W_{t-1} \right) \right\|_F^{\frac 1 2}\right]}{  1 + 2n^{\frac 1 4} \sqrt{\frac{\eta L}{1-\mu_2}} + \frac{\sqrt{\sigma}}{B^{\frac 1 4}}  }\right\}\nonumber \\
& \leq f\left( \boldsymbol W_0 \right) - f^\star + 2 n \epsilon^{-5/4} \eta \sum_{t=1}^T \mathbb E \left[ \left\| \nabla f\left( \boldsymbol W_{t-1} \right) - \boldsymbol M_t  \right\|_F\right] + \frac L 2n \eta^2 T \nonumber \\
& \leq f\left( \boldsymbol W_0 \right) - f^\star + 2 n \epsilon^{-5/4} \eta \sum_{t=1}^T \left( \sqrt{ \frac{1-\mu_1}{1+\mu_1}}\frac{\sigma}{\sqrt B} + \mu_1^t \left\|  \nabla f\left( \boldsymbol W_0 \right)\right\|_F + \frac{ \mu_1 L\eta \sqrt n }{1-\mu_1}   \right) +  \frac L 2n \eta^2 T \nonumber \\  
& \leq  f\left( \boldsymbol W_0 \right) - f^\star + 2 n \epsilon^{-5/4} \eta T \sqrt{ \frac{1-\mu_1}{1+\mu_1}}\frac{\sigma}{\sqrt B} + \frac{  2 n \epsilon^{-5/4} \eta }{ 1-\mu_1 } \left\|  \nabla f\left( \boldsymbol W_0 \right)\right\|_F + \frac{ 2\epsilon^{-5/4}\mu_1Ln^{\frac 3 2} \eta^2 T  }{ 1 - \mu_1 } +  \frac L 2n \eta^2 T 
\end{align}
Thus, we can obtain
\begin{align}\label{eta-condition}
& \min_{t = 1,\cdots, T}    \min \left\{ \mathbb E \left[ \left\| \nabla f\left( \boldsymbol W_{t-1} \right) \right\|_F^{\frac 1 2}\right] , \frac{ \mathbb E^2 \left[ \left\| \nabla f\left( \boldsymbol W_{t-1} \right) \right\|_F^{\frac 1 2}\right]}{  1 + 2n^{\frac 1 4} \sqrt{\frac{\eta L}{1-\mu_2}} + \frac{\sqrt{\sigma}}{B^{\frac 1 4}}  }\right\}  \leq \frac{ 4 \epsilon^{-1/4} \Delta_{0,1} }{T\eta} + 2{\epsilon^{-1/4}L n\eta} \nonumber \\
& \qquad \qquad \qquad \qquad \qquad \qquad + 8 \epsilon^{-3/2} n \sqrt{ \frac{1-\mu_1}{1+\mu_1}}\frac{\sigma}{\sqrt B} + \frac{  8 \epsilon^{-3/2} n \Delta_{0,2} }{\left( 1 - \mu_1 \right)T} + \frac{8 \epsilon^{-3/2} \mu_1 L n^{3/2} \eta }{1-\mu_1},
\end{align}
where $\Delta_{0,1}: = f\left(\boldsymbol W_0 \right) - f^\star, \Delta_{0,2}: = \left\| \nabla  f\left(\boldsymbol W_0 \right)\right\|_F$.\\

\noindent For the right hand side of the above result, there is
\begin{align}
\frac{ 4 \epsilon^{-1/4} \Delta_{0,1} }{T\eta} + 2{\epsilon^{-1/4}L n\eta} + \frac{8 \epsilon^{-3/2} \mu_1 L n^{3/2} \eta }{1-\mu_1} & \leq \frac{ 4 \epsilon^{-1/4} \Delta_{0,1} }{T\eta} + \frac{10 \epsilon^{-3/2}  L n^{3/2} \eta }{1-\mu_1} \nonumber \\
& = 4 \epsilon^{-\frac 7 8} \sqrt{ \frac{10 n^{\frac 3 2}L\Delta_{0,1}}{T\left(1 - \mu_1 \right)}  },
\end{align}
where the equality is due to the step size condition $\eta = \frac{\epsilon^{\frac 5 8}}{n^{\frac 3 4}} \sqrt{ \frac{2\Delta_{0,1}\left(1- \mu_1 \right)}{5LT} }$. Then the \eqref{eta-condition} becomes
\begin{align}
 \min_{t = 1,\cdots, T}    \min \left\{ \mathbb E \left[ \left\| \nabla f\left( \boldsymbol W_{t-1} \right) \right\|_F^{\frac 1 2}\right] , \frac{ \mathbb E^2 \left[ \left\| \nabla f\left( \boldsymbol W_{t-1} \right) \right\|_F^{\frac 1 2}\right]}{  1 + 2n^{\frac 1 4} \sqrt{\frac{\eta L}{1-\mu_2}} + \frac{\sqrt{\sigma}}{B^{\frac 1 4}}  }\right\} & \leq  4 \epsilon^{-\frac 7 8} \sqrt{ \frac{10 n^{\frac 3 2}L\Delta_{0,1}}{T\left(1 - \mu_1 \right)}  } + 8 \epsilon^{-3/2} n \sqrt{1-\mu_1}\frac{\sigma}{\sqrt B} \nonumber \\
& \quad + \frac{  8 \epsilon^{-3/2} n \Delta_{0,2} }{\left( 1 - \mu_1 \right)T}.
\end{align}
Now set $1-\mu_1= \min \left\{  \epsilon^{\frac 5 8}  \sqrt{ \frac{ 5 LB \Delta_{0,1}}{2T n^{\frac 1 2}\sigma^2}} , 1 \right\}$. If  $\epsilon^{\frac 5 8}  \sqrt{ \frac{ 5 LB \Delta_{0,1}}{2T n^{\frac 1 2}\sigma^2}} \leq 1$, then the right hand side of above upper bound becomes
\begin{align}
4 \epsilon^{-\frac 7 8} \sqrt{ \frac{10 n^{\frac 3 2}L\Delta_{0,1}}{T\left(1 - \mu_1 \right)}  } + 8 \epsilon^{-3/2} n \sqrt{1-\mu_1}\frac{\sigma}{\sqrt B} + \frac{  8 \epsilon^{-3/2} n \Delta_{0,2} }{\left( 1 - \mu_1 \right)T} \leq 16 \epsilon^{-\frac{19}{16}} n^{\frac 7 8} \sqrt[4]{ \frac{L\Delta_{0,1}\sigma^2}{TB}} + 8 \epsilon^{-\frac{17}{8}} n ^{\frac 5 4} \sqrt{\frac{\Delta^2_{0,2} \sigma^2}{TLB\Delta_{0,1}}}. 
\end{align}
If $\epsilon^{\frac 5 8}  \sqrt{ \frac{ 5 LB \Delta_{0,1}}{2T n^{\frac 1 2}\sigma^2}} \geq 1$, \eqref{eta-condition}  becomes
\begin{align}
 \min_{t = 1,\cdots, T}    \min \left\{ \mathbb E \left[ \left\| \nabla f\left( \boldsymbol W_{t-1} \right) \right\|_F^{\frac 1 2}\right] , \frac{ \mathbb E^2 \left[ \left\| \nabla f\left( \boldsymbol W_{t-1} \right) \right\|_F^{\frac 1 2}\right]}{  1 + 2n^{\frac 1 4} \sqrt{\frac{\eta L}{1-\mu_2}} + \frac{\sqrt{\sigma}}{B^{\frac 1 4}}  }\right\}  & \leq 16 \epsilon^{-\frac 7 8} n ^{\frac 3 4} \sqrt{\frac{L \Delta_{0,1}}{T}} + 8 \epsilon^{-\frac 3 2} n \frac{\sigma}{\sqrt B} + \frac{8 \epsilon^{-\frac 3 2} n \Delta_{0,2}}{T} \nonumber \\
 & \leq  31 \epsilon^{-\frac 7 8} n ^{\frac 3 4} \sqrt{\frac{L \Delta_{0,1}}{T}}  + \frac{8 \epsilon^{-\frac 3 2} n \Delta_{0,2}}{T},
\end{align}
where the last inequality is due to $  \sigma \leq \epsilon^{\frac 5 8}  \sqrt{ \frac{ 5 LB \Delta_{0,1}}{2T n^{\frac 1 2}}}  $.\\

Finally, combining the results in the above two cases can obtain the \eqref{final-result}.

\begin{lemma}\label{lemma-2}
If the loss function satisfies Assumption \ref{assum-2}, then there is 
\begin{align}
\mathbb E \left[  \left\|  \nabla f \left(\boldsymbol W_{t-1} \right) - \boldsymbol M_t \right\|_F\right] \leq \sqrt{ \frac{1-\mu_1}{1+\mu_1}}\frac{\sigma}{\sqrt B} + \mu_1^t \left\|  \nabla f\left( \boldsymbol W_0 \right)\right\|_F + \frac{ \mu_1 L\eta \sqrt n }{1-\mu_1}.
\end{align}
\end{lemma}
\begin{proof}
The proof idea follows that of \cite{shen2025convergence} and the difference is that we consider the zero initialization. We construct $\boldsymbol C_t = \mu_1 \boldsymbol C_{t-1} + \left( 1-\mu_1\right) \nabla f \left(  \boldsymbol W_{t-1}\right),\forall t \geq 1$ and $\boldsymbol C_0 = \boldsymbol 0_{m \times n}$.  
\begin{align}
\mathbb E \left[  \left\|   \nabla f \left(\boldsymbol W_{t-1} \right) - \boldsymbol M_t   \right\|_F  \right]  & \leq \mathbb E \left[  \left\| \boldsymbol C_t  - \boldsymbol M_t \right\|_F  \right]   + \mathbb E\left[  \left\| \nabla f \left(  \boldsymbol W_{t-1}\right)  - \boldsymbol C_t \right\|_F\right] \nonumber \\
& = \mathbb E \left[   \left\| \left(1 - \mu_1 \right) \sum_{\tau=1}^t \mu_1^{t-\tau}\left( \nabla f \left( \boldsymbol W_{\tau-1}\right) - \boldsymbol G_\tau\right)    \right\|_F     \right] \nonumber \\
& \quad + \mathbb E \left[ \left\| \nabla f \left( \boldsymbol W_{t-1} \right) - \mu_1 \boldsymbol C_{t-1} - \left( 1- \mu_1 \right)\nabla f\left( \boldsymbol W_{t-1}\right) \right\|_F \right] \nonumber \\
& {\leq } \sqrt{ \frac{1-\mu_1}{1+\mu_1}}\frac{\sigma}{\sqrt B} + \mu_1 \mathbb E \left[ \left\|  \nabla f\left( \boldsymbol W_{t-2} \right)- \boldsymbol C_{t-1}\right\|_F \right]  \nonumber \\
& \quad + \mu_1 \mathbb E\left[  \left\| \nabla f\left( \boldsymbol W_{t-1} \right) - \nabla f\left( \boldsymbol W_{t-2} \right) \right\|_F \right] \nonumber \\
& \leq \sqrt{ \frac{1-\mu_1}{1+\mu_1}}\frac{\sigma}{\sqrt B}  + \mu_1^{t-1} \mathbb E \left[ \left\|\nabla f\left( \boldsymbol W_0\right) - \boldsymbol C_1  \right\|_F \right] + \frac{ \mu_1 L\eta \sqrt n }{1-\mu_1} \nonumber \\
& = \sqrt{ \frac{1-\mu_1}{1+\mu_1}}\frac{\sigma}{\sqrt B} + \mu_1^t \left\|  \nabla f\left( \boldsymbol W_0 \right)\right\|_F + \frac{ \mu_1 L\eta \sqrt n }{1-\mu_1},
\end{align}
where the second inequality follows the proof of Lemma A.4 of \cite{shen2025convergence}, and the last equality uses the $\boldsymbol C_1 = \left(1 - \mu_1 \right) \nabla f\left( \boldsymbol W_0 \right)$ due to $\boldsymbol C_0 = \boldsymbol 0_{m \times n}$. 
\end{proof}

\section{Simulation Setting}\label{app:simulation-setting}
We use a matrix quadratic regression problem to illustrate the effect of curvature anisotropy on Muon and the benefit of preconditioning. Specifically, we consider
\begin{align}\label{matrix-quad}
\min_{\boldsymbol X \in\mathbb R^{r\times r}}
f(\boldsymbol G)
:=
\frac{1}{2}
\left\|
\boldsymbol A\boldsymbol X\boldsymbol B^\top -
\boldsymbol C
\right\|_F^2 ,
\end{align}
where $n=50,r=5,\boldsymbol A,\boldsymbol B\in\mathbb R^{n\times r}$. In vectorized form, the Hessian with respect to $\operatorname{vec}(\boldsymbol X)$ is $\boldsymbol H_X =(\boldsymbol B^\top\boldsymbol B) \otimes (\boldsymbol A^\top\boldsymbol A),
$ and therefore its condition number satisfies $
\kappa_X = 
\operatorname{cond}(\boldsymbol H_X) = 
\operatorname{cond}(\boldsymbol B)^2
\operatorname{cond}(\boldsymbol A)^2 .$
To generate controlled-condition problems, we first construct a rank-$r$ matrix $\boldsymbol C
= \boldsymbol U^\star \boldsymbol S^\star \boldsymbol V^{\star^\top}$, where $\boldsymbol U_\star$ and $\boldsymbol V_\star$ are orthonormal factors and $\boldsymbol S^\star$ is a normalized random core matrix. We then set $\boldsymbol A =
\boldsymbol U^\star \boldsymbol R_1, \boldsymbol B = \boldsymbol V^\star \boldsymbol R_2$, where $\boldsymbol R_1,\boldsymbol R_2\in\mathbb R^{r\times r}$ are random matrices with prescribed condition number.
We test four condition numbers $\kappa_X \in \{10^1,10^2,10^3,10^4\}$.

For Figure~\ref{compare-1}, we use the exact polar version of Muon without nuclear-norm scaling. The optimizer used in Figure~\ref{compare-2} is the exact-Hessian preconditioned Muon update. For the quadratic objective above,  the exact whitening preconditioners are chosen as $\boldsymbol P = (\boldsymbol A^\top\boldsymbol A)^{-1/2},
\boldsymbol Q = (\boldsymbol B^\top\boldsymbol B)^{-1/2}$.
Thus, the update is 
\begin{align}
\boldsymbol X_{t+1}
=  \boldsymbol X_t  - \eta_t
(\boldsymbol A^\top\boldsymbol A)^{-1/2}
\operatorname{Orth}
\left(
(\boldsymbol A^\top\boldsymbol A)^{-1/2}
\nabla f\left( \boldsymbol X_t\right)
(\boldsymbol B^\top\boldsymbol B)^{-1/2}
\right)
(\boldsymbol B^\top\boldsymbol B)^{-1/2}.
\end{align} 

Both simulations use the same initialization $\boldsymbol X_0=\boldsymbol 0$, the same maximum number of iterations $T=5000$, and the same warmup-stable-decay learning rate schedule. The learning rate linearly warms up to $\eta_{\max}=10^{-3}$ during the first $500$ iterations, remains constant for $3500$ iterations, and then linearly decays to $\eta_{\min}=0$ during the last $1000$ iterations. For each setting, we report the optimality gap $f(\boldsymbol X_t)-f^\star$,
where $f^\star$ is computed using the closed-form least-squares solution. 

\begin{figure}[htbp]
\centering
\subfigure[Muon]{
\begin{minipage}[t]{0.5\linewidth}
\centering
\includegraphics[width=3.5in]{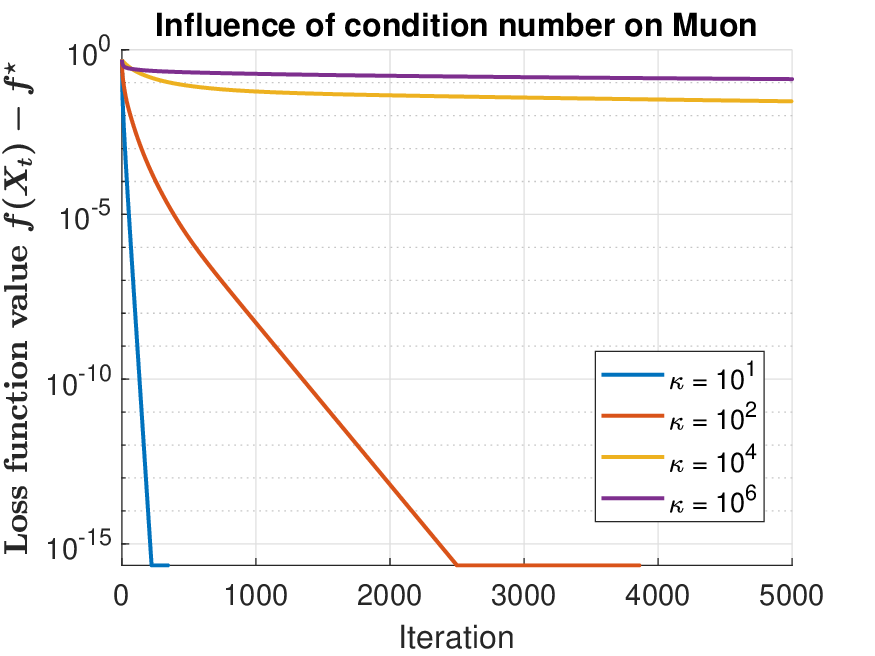}
\label{extra-1}
\end{minipage}%
}%
\subfigure[Preconditioned Muon]{
\begin{minipage}[t]{0.5\linewidth}
\centering
\includegraphics[width=3.5in]{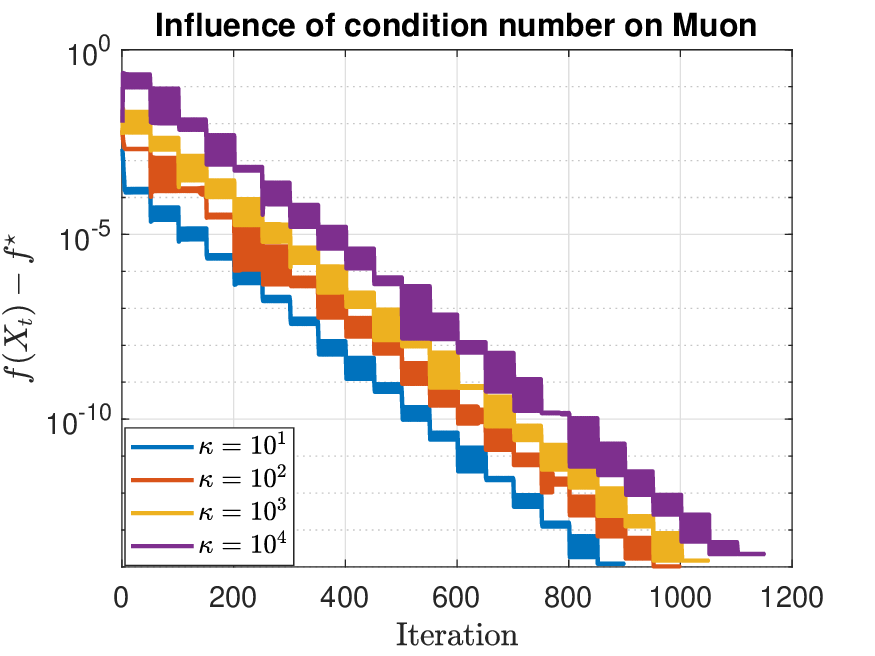}
\label{extra-2}
\end{minipage}%
}%
\centering
\caption{Additional step size settings for Muon-style updates.
(a) Nuclear-norm scaled polar update.
(b) Exact polar Muon with geometric step size decay.}
\label{fig:extra-stepsize}
\end{figure}

To further examine whether the observed sensitivity of Muon is caused by a particular learning rate schedule, we also test two additional learning rate schedules for Muon-style updates on the same matrix quadratic regression problem \eqref{matrix-quad}. The first setting uses a nuclear-norm scaled polar update \cite{lau2025polargrad} as
\begin{align}
\boldsymbol G_{t+1}
=
\boldsymbol G_t
-
\eta_t
\left\|
\nabla f(\boldsymbol G_t)
\right\|_*
\boldsymbol O_t,
\end{align}
which uses constant base step size $
\eta_t
=
\eta_0
=
\frac{C_\eta}{L_G r},
$, where $L_G$ is the largest eigenvalue of the Hessian with respect to $\operatorname{vec}(\boldsymbol G)$ and $C_\eta=2$.

The second setting uses an exact polar Muon direction with a periodically decayed geometric learning rate schedule, where the step size is periodically decayed as $
\eta_t = \eta_0
\rho^{\left\lfloor (t-1)/p \right\rfloor}$, where $p$ is the decay period and $\rho$ is the decay factor. In our implementation, we set $p=50$ and $\rho=0.5$. The initial stepsize is chosen according to the natural scale of the least-squares solution:$ \eta_0 = C_\eta \frac{\|\boldsymbol G_{\rm lsq}\|_F}{\sqrt r}$, where $C_\eta=0.5$ and $\boldsymbol G_{\rm lsq}$ is the closed-form least-squares solution.

Figure~\ref{fig:extra-stepsize} further reports Muon-style updates under these two learning rate schedules. In Figure~\ref{extra-1}, we can observe that the nuclear-norm scaled polar update is still clearly sensitive to the Hessian condition number: as $\kappa$ increases, the convergence becomes much slower. This observation is consistent with Theorem~3.2 in \cite{lau2025polargrad}, which predicts a slower convergence rate when the Hessian condition number becomes larger.  In contrast, Figure~\ref{extra-2} shows that Muon with geometric stepsize decay is much less sensitive to $\kappa$, as the curves exhibit similar staircase-like decreasing trends across different condition numbers.

These results suggest that the sensitivity of Muon to ill-conditioned curvature of loss \eqref{matrix-quad} depends on the learning rate schedule. We use the WSD schedule in Figure \ref{fig:compare-muon-precondition} because WSD is commonly used and studied in modern language-model training. Therefore, the main WSD-based comparison is more aligned with the training protocol used in deep learning and large-scale language-model pretraining.

\newpage
\section{LLM Pretraining Hyperparameter Settings}
\label{app:exp}

\begin{table}[htbp]
\centering
\renewcommand{\arraystretch}{1.2}
\begin{tabular}{l cccc} 
\toprule
\multirow{2}{*}{\textbf{Hyperparameters}} & \multicolumn{4}{c}{\textbf{Optimizers}} \\
\cmidrule(lr){2-5}
& \textbf{AdamW} & \textbf{Muon} & \textbf{MALT} & \textbf{MALTER} \\
\midrule

\multicolumn{5}{l}{\textbf{Small (124M)}} \\
\midrule
Batchsize and Accumulation & \multicolumn{4}{c}{(60, 8)} \\
Weight Decay               & \multicolumn{4}{c}{0.1} \\
Momentum ($\beta_1 / \mu_1, \beta_2 / \mu_2$) & (0.9, 0.95) & (0.95, ---) & (0.95, 0.99) & (0.95, 0.99) \\
Damping Constant ($\epsilon$)                 & ---         & ---         & $10^{-8}$    & $10^{-8}$ \\
\midrule

\multicolumn{5}{l}{\textbf{Medium (355M)}} \\
\midrule
Batchsize and Accumulation & \multicolumn{4}{c}{(40, 12)} \\
Weight Decay               & \multicolumn{4}{c}{0.1} \\
Momentum ($\beta_1 / \mu_1, \beta_2 / \mu_2$) & (0.9, 0.95) & (0.95, ---) & (0.95, 0.99) & (0.95, 0.99) \\
Damping Constant ($\epsilon$)                 & ---         & ---         & $10^{-8}$    & $10^{-8}$ \\
\midrule

\multicolumn{5}{l}{\textbf{Large (774M)}} \\
\midrule
Batchsize and Accumulation & \multicolumn{4}{c}{(20, 24)} \\
Weight Decay               & \multicolumn{4}{c}{0.1} \\
Momentum ($\beta_1 / \mu_1, \beta_2 / \mu_2$) & (0.9, 0.95) & (0.95, ---) & (0.95, 0.99) & (0.95, 0.99) \\
Damping Constant ($\epsilon$)                 & ---         & ---         & $10^{-8}$    & $10^{-8}$ \\

\bottomrule
\end{tabular}
\vspace{0.3cm}
\caption{Hyperparameter configurations for GPT-2 models across different optimizers.}
\label{tab:additional_settings}
\end{table}

\end{document}